\documentclass[10pt,journal,twocolumn,twoside]{IEEEtran} 
\usepackage{cuted}            
\usepackage{color}            
\usepackage{cite}
\usepackage{bm}               
\usepackage{array}            
\usepackage{graphicx}         
\usepackage{algorithm,algorithmic}        
\usepackage{amsmath, amssymb, amsfonts, mathrsfs}   
\usepackage{amsthm}
\usepackage{booktabs}         
\usepackage{stfloats}         
\usepackage{fixltx2e}        
\usepackage{cases}            
\usepackage{mathtools}        
\usepackage{footmisc}         
\usepackage{epstopdf}
\usepackage[T1]{fontenc}      
\usepackage{titlesec}
\usepackage{makecell}         
\usepackage{subfig}           
\usepackage{caption}          
\usepackage{balance}
\usepackage{diagbox}
\usepackage{url}
\usepackage{pifont}
\newtheorem{lemma}{\textit{Lemma}}

\newcommand\figref[1]{Fig.~\ref{#1}}

\newcommand{\tabref}[1]{Table~\ref{#1}}
\newcommand{\sectionref}[1]{Section~\ref{#1}}
\newcommand{\algref}[1]{Algorithm~\ref{#1}}

\begin{document}

\title{Goal-Oriented Semantic Communication for ISAC-Enabled Robotic Obstacle Avoidance}

\author{\IEEEauthorblockN{Wenjie Liu, Yansha Deng, \IEEEmembership{Senior Member, IEEE,} and Henk Wymeersch, \IEEEmembership{Fellow, IEEE}}

\thanks{\copyright 2026 IEEE. Personal use of this material is permitted. Permission from IEEE must be obtained for all other uses, in any current or future media, including reprinting/republishing this material for advertising or promotional purposes, creating new collective works, for resale or redistribution to servers or lists, or reuse of any copyrighted component of this work in other works.}
\thanks{Wenjie Liu and Yansha Deng are with the Department of Engineering, King’s College London, Strand, London WC2R 2LS, U.K. (e-mail: wenjie.liu@kcl.ac.uk; yansha.deng@kcl.ac.uk) (Corresponding author: Yansha Deng).}
\thanks{Henk Wymeersch is with the Department of Electrical Engineering, Chalmers University of Technology, Gothenburg, Sweden (e-mail: henkw@chalmers.se).}
}
	
\maketitle

\begin{abstract}
Obstacle avoidance is a fundamental task in mobile robotics and has been extensively studied over the past decades. However, existing studies are fundamentally limited by an exclusive reliance on robot's onboard sensors, which restricts the field of view and lacks the global understanding of dynamic environments. How to leverage the base station (BS) to enable sensing and control of mobile robots for reliable obstacle avoidance remains largely underexplored. To fill this gap, we investigate an integrated sensing and communication (ISAC)-enabled BS for the unmanned aerial vehicle (UAV) obstacle avoidance task, and propose a goal-oriented semantic communication (GOSC) framework for the BS to transmit sensing and command and control (C\&C) signals efficiently and effectively. Our GOSC framework establishes a closed loop for sensing--C\&C generation--sensing and C\&C transmission: 
For sensing, a Kalman filter (KF) is applied to continuously predict UAV positions, mitigating the reliance of UAV position acquisition on continuous sensing signal transmission, and enhancing position estimation accuracy through sensing–prediction fusion. 
Based on the refined estimation position provided by the KF, we develop a Mahalanobis distance-based dynamic window approach (MD-DWA) to generate precise C\&C signals under uncertainty, in which we derive the mathematical expression of the minimum Mahalanobis distance required to guarantee collision avoidance.
Finally, for efficient sensing and C\&C signal transmission, we propose an effectiveness-aware deep Q-network (E-DQN) to determine the transmission of sensing and C\&C signals based on their value of information (VoI). The VoI of sensing signals is quantified by the reduction in uncertainty entropy of UAV’s position estimation, while the VoI of C\&C signals is measured by their contribution to UAV navigation improvement. 
Extensive simulations validate the effectiveness of our proposed GOSC framework. Compared to the conventional ISAC transmission framework that transmits sensing and C\&C signals at every time slot, GOSC achieves the same 100\% task success rate while reducing the number of transmitted sensing and C\&C signals by 92.4\% and the number of transmission time slots by 85.5\%. 
\end{abstract}

\section{Introduction}
\IEEEPARstart{O}{bstacle} avoidance is one of the fundamental tasks in mobile robotics, which aims at navigating a robot towards its destination without collisions in environments populated with obstacles \cite{OA}. Over the past decades, it has been extensively studied due to its broad applications in areas such as autonomous vehicles, unmanned aerial vehicles (UAVs), service robotics, and industrial automation. Broadly, existing studies on robotic obstacle avoidance can be classified into two categories: multi-sensor fusion approaches \cite{sensor_fusion} and collision-avoidance algorithmic approaches \cite{OA_alg}. The former focuses on integrating data from multiple onboard sensors--such as cameras, inertial measurement units (IMUs), and LiDARs--in efficient ways to provide richer environmental information that supports navigation decision-making \cite{sensor}. The latter emphasizes the  design of more sophisticated obstacle avoidance algorithms, leveraging techniques from control theory, optimization, and artificial intelligence (AI) to enhance navigation safety and efficiency \cite{OA_AI}.

However, a common limitation of existing mobile robot obstacle avoidance research is that the environmental perception is based solely on robot's onboard sensors, which inherently leads to two drawbacks: a restricted sensing range of each robot that prevents global environmental awareness, and the challenge in effective coordination among multiple robots via global view. Integrated sensing and communication (ISAC) \cite{ISAC, reviewer4_1, reviewer4_2} is a promising technique to overcome these two drawbacks. With ISAC, the base station (BS) can act as a global coordinator of multiple robots, providing a global view and control. This can be achieved by transmitting sensing signals to detect the positions of robots and environmental obstacles; based on the sensed positions, the BS could subsequently generate and transmit command-and-control (C\&C) signals via downlink communication to specify the robots' movement. 

As an emerging technique, ISAC has attracted increasing attention in recent years. However, existing studies on ISAC is primarily limited to the sensing and single link communication performance optimization, e.g., channel capacity \cite{capacity}, energy consumption \cite{energy}, mean-square-error (MSE) \cite{mse}, Cram\'{e}r-Rao bound (CRB) \cite{CRB}, and the inherent trade-off between sensing and communication \cite{trade-off}. The research question of ``How should the BS jointly design sensing and C\&C signals to control robots for safe and reliable obstacle avoidance'' remains largely unexplored. A commonly adopted approach in the existing ISAC literature is to transmit ISAC signals continuously \cite{continuous}. Nevertheless, this strategy presents a critical drawback: ensuring safe navigation requires the ISAC signals to be transmitted at each time slot with a high frequency, which results in high communication and computation costs. Notably, when there are no obstacles in the vicinity of the robots, continuous ISAC transmission yields redundant information and unnecessary resource consumption. 

Goal-oriented semantic communication (GOSC) \cite{GOSC} has recently emerged as a promising paradigm to address the problem of redundant data transmission. Different from deep joint source and channel coding (JSCC) frameworks \cite{JSCC, Reviewer2}, which typically rely on end-to-end deep neural network (DNN) training and latent feature embeddings, these frameworks may suffer from poor interpretability and sensitivity to wireless channel variations. GOSC effectively mitigates these limitations by extracting and transmitting only the semantic representation that directly contributes to the application objective \cite{SR1, SR2}. Owing to these advantages, recent efforts have explored the application of GOSC in robotic control \cite{xu, wu}. It is worth noting that, for robotic control applications, the semantic representation corresponds to the critical control information required for motion control, such as thrust, roll angle, yaw angle, and velocity of the C\&C data \cite{xu}.
 Xu \textit{et al.} \cite{xu} defined semantic-level and effectiveness-level performance metrics for C\&C signals, and designed a general task-oriented semantic-aware framework to reduce the redundant transmission of C\&C signals. Wu \textit{et al.} \cite{wu} proposed a joint age of information (AoI) and value of information (VoI) based queue ranking strategy to prioritize those effectiveness-critical C\&C signals received by the UAV. However, \cite{xu} and \cite{wu} focused solely on the transmission of C\&C signals under the ideal assumption that the BS knows the accurate location of UAV in real-time. This assumption is unrealistic in practice since sensing inevitably introduces detection errors. While in the context of ISAC, the transmission of C\&C signals inherently depends on the accuracy of sensing results. This intrinsic coupling between sensing and communication highlights the necessity to jointly design the GOSC transmission for both sensing and C\&C signals. 

Motivated by the above, the main contributions of this work are summarized as follows:
\begin{itemize}
    \item We investigate an ISAC-enabled BS for UAV obstacle avoidance task from a GOSC perspective. Different from JSCC-based end-to-end semantic communication frameworks, our work adopts GOSC that operates at the both semantic and effectiveness levels, addressing \emph{when} and \emph{what} semantic representations should be transmitted to optimize task performance. The objective is to enable the UAV to safely and quickly reach its destination while minimizing the total number of transmitted sensing and C\&C signals. To the best of our knowledge, this is the first work to investigate efficient integrated sensing and C\&C signaling codesign in ISAC systems from a GOSC perspective.
    \item We propose a unified GOSC framework to achieve the objective, which jointly integrates state estimation, C\&C signal generation, and transmission scheduling under a VoI-driven semantic decision structure. Specifically, the framework forms a closed loop of sensing--C\&C generation--adaptive transmission, enabling efficient and task-oriented signaling:
\begin{enumerate}  
    \item For sensing, we apply a Kalman filter (KF) to continuously predict the UAV's position based on the most recently transmitted C\&C signals, which enables the BS to acquire a reference position of the UAV without the needs to transmit sensing signals. When a sensing signal is transmitted, it can also fuse sensing and prediction results to provide a more accurate position estimation of the UAV.
    \item Based on the UAV position provided by the KF, we propose a Mahalanobis distance-based dynamic window approach (MD-DWA) to generate precise C\&C signals under estimation uncertainty. In particular, we derive a tractable mathematical expression for the minimum Mahalanobis distance under a fixed safety constraint, enabling simultaneous consideration of probabilistic uncertainty and deterministic collision avoidance.
    \item We design an effectiveness-aware deep Q-network (E-DQN) to determine whether to transmit sensing and C\&C signals based on their VoI. Unlike conventional metrics such as data size or channel conditions, the VoI is quantified by task relevance--via uncertainty reduction for sensing signal and navigation improvement for C\&C signal--to enable E-DQN to learn transmission policies based on semantic importance.
\end{enumerate}  
    \item Extensive simulations validate the effectiveness of our proposed GOSC framework. Compared to the traditional ISAC transmission framework that transmits sensing and C\&C signals at every time slot, GOSC achieves the same 100\% task success rate while reducing the number of transmitted sensing and C\&C signals by 92.4\% and  the required transmission time slots by 85.5\%.
\end{itemize}

 
The rest of this paper is organized as follows. \sectionref{model} presents the system model and problem formulation; \sectionref{trad} introduces a traditional ISAC signal transmission framework; \sectionref{proposed} describes our proposed GOSC framework; \sectionref{simulation} presents simulation results and analysis; and \sectionref{conclusion} concludes the paper.

\textbf{Notation}. Unless otherwise specified, we denote column vectors as bold lowercase italics (e.g., $\bm{a}$), matrices as bold uppercase italics, (e.g., $\bm{A}$), and constants as uppercase letters, (e.g., $\mathrm{A}$). The $i$-th element of a vector and the $(i,j)$-th element of a matrix are denoted by $[\bm{a}]_i$ and $[\bm{A}]_{i,j}$, respectively.
The Euclidean norm is represented by $\|\cdot\|$, and the Mahalanobis distance by $\|\cdot\|_M$.
The transpose, conjugate, Hermitian, and inverse of a matrix are denoted by $(\cdot)^\top$, $(\cdot)^*$, $(\cdot)^\mathrm{H}$ and $(\cdot)^{-1}$, respectively.
The Gaussian distribution is denoted as $\mathcal{N}$, while the complex Gaussian distribution is denoted as $\mathcal{CN}$.
The operator $\odot$ represents element-wise multiplication.
A column vector of ones is denoted as $\bm{1}$, and a column vector of zeros is denoted as $\bm{0}$.
The identity matrix of size $\mathrm{N}\times\mathrm{N}$ is denoted as $\bm{I}_{\mathrm{N}}$, and a zero matrix of size $\mathrm{N}\times\mathrm{N}$ is denoted as $\bm{0}_{\mathrm{N}}$.



\section{System Model and Problem Formulation} \label{model}
In this section, we provide a brief introduction to the ISAC-enabled BS for the UAV obstacle avoidance task. We then describe the sensing and C\&C signal transmission and reception. Next, the kinematic models of the UAV and obstacles are presented, followed by the beamforming model. Finally, a problem is formulated with the aim to transmit sensing and C\&C signals efficiently.

\subsection{Scenario Description}
As shown in \figref{system_model}, we consider an ISAC-enabled BS for the UAV obstacle avoidance task in open airspace, where the BS sends sensing signals to detect the position of UAV, which is then used to generate and transmit C\&C signals. The C\&C signals specify UAV's speed and heading angle to guide it flying towards the destination while avoiding dynamic obstacles. 
In this context, dynamic obstacles refer to moving objects (e.g., other UAVs or birds) in the sky whose speeds/positions vary over time and may potentially interfere with the UAV’s trajectory.
The whole task period can be discretized into $I$ time slots, and the time slot index is denoted as $t_i$, where $i\in\{0, 1,2,..., I\}$ and the time slot interval $\Delta t = t_{i+1}-t_i~(\forall i)$ is a constant. Without loss of generality, we assume the UAV flies in the $x-y$ plane at a fixed altitude $\mathrm{H}$. The real horizontal position of UAV at $t_i$ is denoted as $\bm{p}(t_i) = [p_x(t_i),~p_y(t_i)]^\top$. Note that the BS can not obtain the real position of UAV $\bm{p}(t_i)$, and can only send sensing signal to detect the position of UAV. The detected horizontal position is denoted as $\hat{\bm{p}}(t_i)=[\hat{p}_x(t_i),~\hat{p}_y(t_i)]^\top$. We assume that the horizontal positions of obstacles are detected by other cooperative BSs. All BSs are interconnected via high-capacity fiber backhaul, enabling low-latency information exchange and reliable synchronization. Consequently, the delay and synchronization errors introduced by backhaul communication are assumed negligible and do not affect the UAV control process. The real and detected obstacle positions are denoted as $\bm{p}_o(t_i)$ and $\hat{\bm{p}}_o(t_i)$, respectively, where $o\in\{1,2,3,..., N_o\}$ and $N_o$ is the total number of obstacles. Due to the thermal noise and quantization noise, there are always detection errors between the real position and the detected position, both for the UAV and obstacles \cite{dect_errors}. Under such detection errors, our goal is to design an efficient and reliable joint sensing and C\&C transmission scheme that enables the UAV to avoid dynamic obstacles and reach its destination as quickly as possible.
\begin{figure}[t] 
    \centering
    \includegraphics[width=0.43\textwidth]{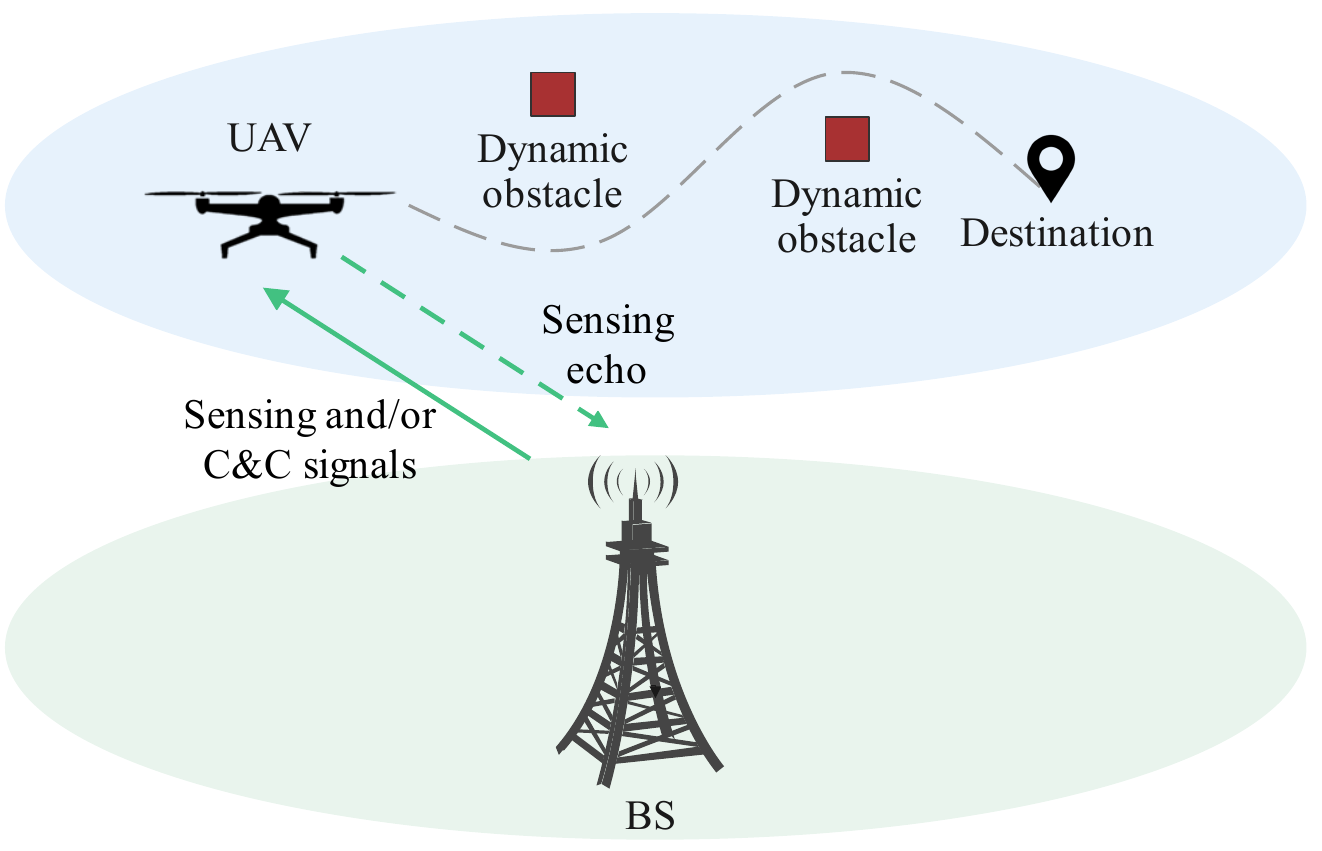}
    \caption{ISAC-enabled BS for the UAV obstacle avoidance task in open airspace.}
    \label{system_model}
\end{figure}

\subsection{Sensing and C\&C Signal Transmission Models}
We consider a mono-static multiple-input multiple-output (MIMO) system, in which the BS uses the same antenna array to transmit millimeter wave (mmWave) sensing and/or C\&C signals and receive sensing echoes. We assume that the BS is equipped uniform linear array (ULA) with $\mathrm{K}$ antennas, which transmit orthogonal frequency division multiplexing (OFDM) signal across $\mathrm{M}$ subcarriers with a subcarrier spacing $\Delta f$. The antenna element spacing is $\mathrm{D}=\lambda_c/2$, where $\lambda_c$ is the carrier wavelength. The OFDM signal has a symbol duration of $\mathrm{T}_{\text{sym}} = \mathrm{T}_{\text{cp}} + \mathrm{T_c}$, where $\mathrm{T}_{\text{cp}}$ is the cyclic prefix (CP) duration and $\mathrm{T_c} = 1/\Delta f$ is the elementary OFDM symbol duration. The complex baseband OFDM signal for sensing at instant time $t$ within the $t_i$-th time slot can be denoted as
\begin{equation}
    \mathrm{x}_i^s(t) = \frac{1}{\sqrt{\mathrm{M}}} \sum_{m=0}^{\mathrm{M}-1}z_{m}^se^{j2\pi m\Delta f(t-t_i)}\text{rect}\Big(\frac{t-t_i}{\mathrm{T_{sym}}}\Big),
\end{equation}
where $z_{m}^s$ is the complex sensing symbol mapped on the $m$-th subcarrier and $\mathbb{E}\{|z^s_m|^2\}=1$, $\text{rect}(\cdot)$ is the rectangular function and is defined as 
\begin{equation}
    \text{rect}(x) \triangleq \begin{cases}
        1, \quad  0 \leq x \leq 1, \\
        0, \quad \text{otherwise}.
    \end{cases}
\end{equation}

We assume that the C\&C data is modulated into $\mathrm{N_{cc}}$ symbols via quadrature phase shift keying (QPSK). It should be mentioned that sensing typically requires a large bandwidth to achieve high-resolution detection, while transmitting the C\&C signals only consumes a rather small bandwidth. In other words, $\mathrm{M}$ is much greater than $\mathrm{N_{cc}}$. Thus, different from conventional OFDM mapping where each symbol is assigned to a single subcarrier, we adopt a redundancy-enhanced frequency-domain repetition scheme to improve reliability. The complex downlink baseband OFDM signal at instant time $t$ within the $t_i$-th time slot can be denoted as 
\begin{equation}  \label{com_signal}
\begin{aligned}
    \mathrm{x}_i^c(t) = &\frac{1}{\sqrt{\mathrm{M}}} \sum_{m=0}^{\mathrm{M}-1}\sum_{n=0}^{\mathrm{N_{cc}-1}} \eta_{m,n} z_{n}^c(t_i)e^{j2\pi m\Delta f(t-t_i)}\text{rect}\Big(\frac{t-t_i}{\mathrm{T_{sym}}}\Big),
\end{aligned}
\end{equation}
where $z_n^c(t_i)$ is the $n$-th complex C\&C symbol at $t_i$ and $\mathbb{E}\{|z_n^c(t_i)|^2\}=1$, $\eta_{m,n} \in \{0,1\}$ denotes whether mapping the $n$-th C\&C symbol on the $m$-th subcarrier. Specifically, each subcarrier $m$ is assigned to at most one symbol $z_n^c(t_i)$, i.e., for each $m$, there exists at most one $n$ such that $\eta_{m,n}=1$. Meanwhile, since $\mathrm{M} \gg \mathrm{N_{cc}}$, each C\&C symbol is mapped onto multiple subcarriers, i.e., $\sum_{m=0}^{\mathrm{M}-1} \eta_{m,n} \gg 1$, which realizes frequency-domain repetition. This design provides frequency diversity and improves robustness against noise and channel impairments. The complex baseband transmit signal at $t_i$ is expressed as 
\begin{equation}
    \mathrm{x}_i^{\text{tx}}(t) = \delta_s(t_i)\mathrm{x}_i^s(t) + \delta_c(t_i)\mathrm{x}_i^c(t),
\end{equation}
where $\delta_s(t_i)$ and $\delta_c(t_i)$ are binary variables that indicate whether to transmit sensing and communication baseband signals, respectively. For clarity, if the BS transmits both sensing and C\&C signals within one time slot, we denote the combined signal as `ISAC signal'. Here, it is important to note that the generation of C\&C signals depends on the sensing detection results. Without loss of generality, we assume that it takes one time slot for the BS to receive the sensing echo, process the sensing data, and generate the corresponding C\&C signal. The C\&C signal can be transmitted in the subsequent time slot. Thus, $\delta_s(t_i)$ and $\delta_c(t_i)$ are defined as 
\begin{align}
        &\delta_s(t_i) \triangleq \begin{cases}
        1, \quad &\text{transmitting sensing signal at the} \\ 
        &\text{beginning of } t_i, \\
        0, \quad &\text{otherwise},  
    \end{cases}  \\
    &\delta_c(t_i) \triangleq \begin{cases}
        1, \quad &\text{if sensing signal is transmitted at $t_{i-1}$}\\ 
           &\text{and C\&C signal is transmitted at the}  \\
           &\text{beginning of } t_i, \\
        0, \quad &\text{otherwise}.
    \end{cases}
\end{align}

\vspace{-0.35 cm}
\subsection{Sensing Receiver Signal Model}
Practical studies indicate that UAV--BS links are predominantly line-of-sight (LoS) in open-air environments \cite{LOS}. Accordingly, a LoS-dominant sensing channel is adopted\footnote{For more complex fading environments (e.g., Rayleigh or probabilistic LoS channels), the sensing model can be extended to account for multipath components with random channel gains and phases, leading to increased estimation uncertainty. The proposed GOSC framework naturally accommodates this effect, as sensing uncertainty is explicitly modeled (e.g., via the KF covariance) and propagated to decision-making through the Mahalanobis distance and VoI design. Hence, channel fading primarily manifests as degraded estimation accuracy without requiring structural changes to the GOSC framework.}. LoS blockage is not considered in this work, as the UAV and dynamic obstacles are assumed to lie on the same horizontal plane (i.e., at similar altitudes), which prevents obstruction of the direct propagation path. Clutter and multipath effects are mitigated by the BS's directional beamforming toward the sky and the relatively small radar cross sections (RCS) of aerial objects.
When the BS transmits a sensing signal at the beginning of the $t_i$-th time slot, the signal is reflected by the UAV and its surrounding dynamic obstacles, and received by the BS within the same slot. 
After signal sampling and applying an $\mathrm{M}$-point fast Fourier transform (FFT), the received sensing echo at the BS in the frequency domain can be expressed as \cite{sensing_receiver} 
\begin{equation}  \label{sensing_echo}
\begin{aligned}
    \bm{Y}_s(t_i)=&\sum_{l=1}^{L_s(t_i)}\sqrt{\frac{\mathrm{P}}{\mathrm{M}} }{}\xi_l(t_i)\bm{a}_{\text{rx}}(\theta_l(t_i))\bm{a}_{\text{tx}}^\mathrm{H}(\theta_l(t_i))\bm{f}(t_i) \cdot \\
    &\big[\bm{\mathrm{x}}_{s}\odot\bm{\mu}(\tau_s^l(t_i))\big]^\top  
    + \bm{Z},
\end{aligned}
\end{equation}
where $\mathrm{P}$ is the transmit power; 
$L_s(t_i)$ is the total number sensing targets, including the UAV and obstacles within the scan scope of the BS;
$\xi_l(t_i) = \sqrt{\frac{\sigma_{\text{RCS}}^l\lambda_c^2}{(4\pi)^3r_l^4(t_i)}}$ represents the attenuation coefficient, $\sigma_{\text{RCS}}^l$ is the RCS of the $l$-th sensing target, $r_l(t_i)$ is the Euclidean distance from the the $l$-th sensing target to the BS at $t_i$; 
$\bm{f}(t_i) \in \mathbb{C}^{\mathrm{K}\times 1}$ is the beamforming vector and discussed in \textit{Subsection F};
$\bm{\mathrm{x}}_{s} =[z^s_0,z^s_1,...,z^s_\mathrm{M-1}]^\top$;
$\bm{\mu}(\tau_s(t_i)) \in \mathbb{C}^{\mathrm{M} \times 1}$ is the phase shift across OFDM subcarriers, in which each element can be expressed as $[\bm{\mu}(\tau_s^l(t_i))]_m = e^{-j2\pi m\Delta f \tau_s^l(t_i)}$, $\tau_s^l(t_i) = \frac{2r_l(t_i)}{\mathrm{c_0}}$ is the round trip delay;
$\bm{Z} \in \mathbb{C}^{\mathrm{K} \times \mathrm{M}}$ is the additive noise and each entry $[\bm{Z}]_{k,m}\sim \mathcal{CN}(0, \Delta f\sigma_0^2)$, $\sigma_0^2$ is the power spectral density of additive white Gaussian noise (AWGN); 
$\bm{a}_{\text{rx}}(\theta_l(t_i)) \in \mathbb{C}^{\mathrm{K} \times 1}$ and $\bm{a}_{\text{tx}}(\theta_l(t_i)) \in \mathbb{C}^{\mathrm{K} \times 1}$ are steering vectors of receiver and transmitter, respectively, which are denoted as
\begin{equation}
\begin{aligned}
    \bm{a}_{\text{rx}}(\theta_l(t_i))&=\bm{a}_{\text{tx}}(\theta_l(t_i)) \\
    &= \bigg[1, e^{-j2\pi\frac{\mathrm{D}\sin(\theta_l(t_i))}{\lambda_c}}, ..., e^{-j2\pi\frac{(\mathrm{K}-1)\mathrm{D}\sin(\theta_l(t_i))}{\lambda_c}}\bigg]^\top.
\end{aligned}
\end{equation}

It is worth mentioning that, within a single OFDM symbol duration, the Doppler effect manifests as an intra-symbol phase rotation. Although perfect Doppler compensation may not be achievable in practice, the residual phase rotation is negligible and is not included in \eqref{sensing_echo}. 
Since the BS has access to the complex baseband signal $\bm{\mathrm{x}}_{s}$ in advance, we can remove its impact on the sensing receiver signal $\bm{Y}_s(t_i)$ via zero-forcing reciprocal filtering \cite{reciprocal}. The filtered signal is written as  
\begin{equation}
\begin{aligned}
    \widetilde{\bm{Y}}_s(t_i)=&\sum_{l=1}^{L_s(t_i)}\sqrt{\frac{\mathrm{P}}{\mathrm{M}}} \xi_l(t_i)\bm{a}_{\text{rx}}(\theta_l(t_i))\bm{a}_{\text{tx}}^\mathrm{H}(\theta_l(t_i)) \cdot \\
    &\bm{f}(t_i) \bm{\mu}(\tau_s^l(t_i))+ \widetilde{\bm{Z}},
\end{aligned}
\end{equation}
where $\widetilde{\bm{Z}}$ denotes the AWGN noise components after filtering. After obtaining $\widetilde{\bm{Y}}_s(t_i)$, we employ the multiple signal classification (MUSIC) algorithm \cite{MUSIC} and the parabolic interpolation of FFT (PIFFT) method \cite{PIFFT} to estimate the angle-of-arrival (AoA) $\hat{\theta}(t_i)$ and the distance $\hat{r}(t_i)$ of the UAV, respectively. Note that the primary BS is mainly responsible for estimating the position of the UAV. The positions and velocities of obstacles are assumed to be provided by cooperative BSs. Moreover, to ensure safe operation, a minimum safety distance $\mathrm{D}_{\text{safe}}$ is maintained between the UAV and surrounding obstacles. Consequently, the range and angular separations between the UAV and obstacles exceed the range resolution $(c / (2\mathrm{M}\Delta f))$ and angular resolution ($2/\mathrm{K}$) of the primary BS, respectively, such that their echoes are resolvable in the range-angle domain. Moreover, the obstacle positions are identified by cooperative BSs and can be filtered out during MUSIC and PIFFT processing. Therefore, obstacle echoes do not interfere with UAV parameter estimation. The detection variance for the AoA of the UAV $\hat{\theta}(t_i)$ can be expressed as \cite{MUSIC_dev}
\begin{equation} \label{aoa_var}
    \sigma^2(\theta(t_i)) \approx \frac{6}{\mathrm{SNR_s}(t_i)\pi^2\cos^2(\theta(t_i))\mathrm{K}^3},
\end{equation}
where $\mathrm{SNR_s}(t_i)$ is the sensing signal-to-noise ratio (SNR) at $t_i$ and is denoted as 
\begin{equation} \label{snr_sense}
    \mathrm{SNR_s}(t_i)= \frac{\mathrm{P}|\xi(t_i)|^2|\bm{a}_{\text{tx}}^\mathrm{H}(\theta(t_i))\bm{f}(t_i)|^2}{\mathrm{M}\Delta f\sigma^2_0 }.
\end{equation}

The detection variance for the distance of UAV $\hat{r}(t_i)$ is given as \cite{PIFFT_dev}
\begin{equation} \label{range_var}
    \sigma^2(r(t_i)) \approx  \bigg(\frac{\mathrm{c_0}}{2\mathrm
    {M}\Delta f}\bigg)^2\frac{1}{16\pi^2\mathrm{SNR_s}(t_i)}.
\end{equation}

The relationship between the real AoA $\theta(t_i)$, distance $r(t_i)$ and their detected AoA $\hat{\theta}(t_i)$, distance $\hat{r}(t_i)$ are denoted as 
\begin{equation}
    \begin{cases}
    \hat{\theta}(t_i) = \theta(t_i)  + \epsilon_{\theta}(t_i), \\
    \hat{r}(t_i) = r(t_i) + \epsilon_r(t_i), 
    \end{cases}
\end{equation}
where $\epsilon_{\theta}(t_i)\sim\mathcal{N}(0, \sigma^2(\theta(t_i)))$ and $\epsilon_r(t_i)\sim\mathcal{N}(0, \sigma^2(r(t_i)))$ are detection errors of AoA and distance, respectively. The real, detection positions of UAV in 2D Cartesian coordination and their relationship are represented as
\begin{equation}  \label{est_real}
\begin{cases}
        \bm{p}(t_i) = \Big[r(t_i)\cos(\theta(t_i)),~r(t_i)\sin(\theta(t_i))\Big]^\top, \\
        \hat{\bm{p}}(t_i) = \Big[\hat{r}(t_i)\cos(\hat{\theta}(t_i)),~\hat{r}(t_i)\sin(\hat{\theta}(t_i))\Big]^\top, \\
        \hat{\bm{p}}(t_i) = \bm{p}(t_i)  + [\epsilon_x(t_i), \epsilon_y(t_i)]^\top.
\end{cases}
\end{equation} 

\begin{lemma}  \label{est_error}
    The detection errors $\epsilon_x(t_i)$ and $\epsilon_y(t_i)$ in 2D Cartesian coordination can be derived as
    \begin{equation}
        \begin{bmatrix}
            \epsilon_x(t_i) \\
            \epsilon_y(t_i)
        \end{bmatrix} \approx \bm{J}(t_i)
            \begin{bmatrix}
            \epsilon_r(t_i) \\
            \epsilon_{\theta}(t_i)
        \end{bmatrix},
    \end{equation}
    where $\bm{J}(t_i)$ is the Jacobian matrix of $\bm{p}(t_i)$ evaluated at $(\hat{r}(t_i),\hat{\theta}(t_i))$.
\end{lemma}
\begin{proof}
    Please see Appendix A.
\end{proof}

\textit{Remark 1: The Validity of Lemma 1.} \textit{\textbf{Lemma 1}} is derived based on a first-order Taylor expansion that linearizes the nonlinear transformation. The approximation error introduced by the first-order expansion is of higher order (i.e., second-order terms in $\epsilon_r(t_i)$ and $\epsilon_{\theta}(t_i)$). Therefore, the approximation error scales quadratically with the estimation errors and is negligible when these errors are small \cite{first_oder_error}. This condition is satisfied due to the relatively high SNR in LoS-dominant channel and high sensing resolution. Consequently, the linearization error has a marginal effect on the overall system performance.

According to \textit{\textbf{Lemma 1}}, the detection position variances $\sigma^2_x(t_i)$ and $\sigma^2_y(t_i)$ of the UAV in 2D Cartesian coordination can be denoted as 
\begin{equation}
\begin{aligned}
\sigma^2_x(t_i)=&\cos^2\big(\theta(t_{i})\big)\sigma^2\big(r(t_{i})\big) + r^2(t_{i}) \sin^2\big(\theta(t_{i})\big) \times \\ &\sigma^2\big(\theta(t_{i})\big),\\
\sigma^2_y(t_i)=&\sin^2\big(\theta(t_{i})\big)\sigma^2\big(r(t_{i})\big) + r^2(t_{i})\cos^2\big(\theta(t_{i})\big)\times  \\ &\sigma^2\big(\theta(t_{i})\big). \\
\end{aligned}
\end{equation}

\subsection{Communication Receiver Signal Model}
We consider a rotary-wing UAV equipped with one receive antenna. The wireless channel is modeled as 
\begin{equation} \label{com_channal}
    \bm{h}(t_i) = \sqrt{\beta(t_i)}\bm{g}(t_i,\zeta),
\end{equation}
where $\bm{h}(t_i)\in\mathbb{C}^{\mathrm{K}\times 1}$, $\beta(t_i)$ is the large-scale fading, $\bm{g}(t_i,\zeta)\in\mathbb{C}^{\mathrm{K}\times 1}$ is the small-scale fading. Note that the BS and UAV are synchronized in both time and frequency domains. In practice, this can be achieved via standard OFDM synchronization techniques (e.g., pilot-assisted timing and carrier frequency offset estimation). Given the LoS-dominant channel and relatively short communication distance, residual synchronization errors are small and thus neglected in \eqref{com_channal}. Since the wireless channel between the BS and the UAV is dominated by the LoS link, the free-space path loss model \cite{FSPL} is used to describe the large-scale fading
\begin{equation}
    \beta(t_i) = \bigg(\frac{\lambda_c}{4\pi r(t_i)}\bigg)^2.
\end{equation}

Similarly, due to the existence of LoS link, the small-scale fading is modeled by the Rician fading below \cite{LOS}
\begin{equation} \label{rician}
\begin{aligned}
        \bm{g}(t_i,\zeta) =& \sqrt{\frac{\kappa}{\kappa+1}}\bm{a}_{\text{tx}}(\theta(t_i))\alpha_0\delta(\zeta-\zeta_0) \\ 
        &+ \sqrt{\frac{1}{\kappa+1}}\sum_{q=1}^{Q(t_i)-1}\alpha_q\delta(\zeta-\zeta_q),
\end{aligned}
\end{equation}
where $\kappa$ is the Rician factor, $\alpha_{0}$ is the deterministic LoS channel component with $|\alpha_{0}|=1$, $\zeta_0$ is the delay on LoS channel, $\delta(\cdot)$ is Dirac delta function, $\alpha_{q}$ denotes non-line-of-sight (NLoS) fading component, $\alpha_q\sim \mathcal{CN}(0, \sigma_q^2)$ and $\sum_q\sigma_q^2=1$, $\zeta_q$ is the delay on the $q$-th NLoS channel, $Q(t_i)$ is the total number of transmission paths. After Fourier transform, the channel frequency response of $\bm{g}(t_i,\zeta)$ can be expressed
\begin{equation}
    \begin{aligned}
        \bm{G}(t_i,f) =& \sqrt{\frac{\kappa}{\kappa+1}}\bm{a}_{\text{tx}}(\theta(t_i ))\alpha_0e^{-j2\pi f\zeta_0} \\
        &+ \sqrt{\frac{1}{\kappa+1}}\sum_{q=1}^{Q(t_i)-1}\alpha_qe^{-j2\pi f\zeta_q}.
    \end{aligned}
\end{equation}

As noted in equation \eqref{com_signal}, we assume the C\&C symbols are repeatedly mapped on $\mathrm{L}$ disjoint frequency blocks across $\mathrm{M}$ subcarriers. Denote the set of subcarriers used by the $l$-th repetition as $\mathcal{S}_l$ with cardinality $|\mathcal{S}_l|~(\sum_l|\mathcal{S}_l|\leq\mathrm{M})$. At the receiver, after OFDM demodulation, the received signals corresponding to the same symbol $z_n^c(t_i)$ across different subcarriers are combined using maximum ratio combining (MRC)\cite{MRC} to achieve a high SNR gain. The combined signal is then used for standard QPSK symbol detection. Specifically, the communication SNR between the BS and the UAV at the $t_i$-th time slot can be expressed as
\begin{equation}
    \mathrm{SNR_c}(t_i) = \frac{\mathrm{P}\beta(t_i)}{\mathrm{M}\Delta f\sigma^2_0}\sum_{l=1}^{\mathrm{L}}\sum_{f'\in \mathcal{S}_l}\Big|\bm{G}^\top(t_i,f')\bm{f}(t_i)\Big|^2.
\end{equation}

The latency for transmitting the C\&C data is denoted as 
\begin{equation}
\begin{aligned}
    \tau_c(t_i) = \frac{\mathrm{2LN_{cc}}}{\mathrm{M}\Delta f\log_2\big(1 + \mathrm{SNR_c}(t_i)\big)}.
\end{aligned}
\end{equation}

Given that the C\&C data size is relatively small and the LoS-dominant channel provides high-SNR conditions, the transmission latency $\tau_c(t_i)$ is consistently lower than the time slot duration $\Delta t$. Moreover, frequency-domain repetition combined with MRC ensures guarantees reliable decoding at the UAV. For analytical tractability, we assume the UAV possesses sufficient computational resources to process the signal upon reception. Consequently, C\&C signals are assumed to be successfully decoded within the same time slot in which they are transmitted.

\subsection{UAV Kinematic Model}
At every time slot, the forward angle $\phi(t_i)$ and speed $V(t_i)$ of the UAV are controlled by the C\&C signal. Specifically, if the C\&C signal is transmitted at the $t_i$-th time slot, the UAV executes its corresponding command $\phi(t_i)$ and $V(t_i)$; otherwise, the UAV adopts the command from the last time slot $\phi(t_{i-1})$ and $V(t_{i-1})$, which can be written as
\begin{equation} \label{speed}
    \begin{cases}
        V(t_i) = \delta_c(t_i)V(t_i) + (1-\delta_c(t_i))V(t_{i-1}),  \\
        \phi(t_i) = \delta_c(t_i)\phi(t_i) + (1-\delta_c(t_i))\phi(t_{i-1}). \\
    \end{cases}
\end{equation}

Due to propulsion limitations of the UAV, the maximum changes in angle and speed within one time slot are $\Delta \phi$ and $\Delta \mathrm{V}$, respectively; the maximum flying speed can not exceed $\mathrm{V}_{\max}$.  Thus, the commanded values should satisfy
\begin{equation}  \label{cc_cons}
    \begin{cases}
        \big|V(t_i) - V(t_{i-1})\big| \leq \Delta \mathrm{V},~V(t_i)\in [0,\mathrm{V}_{\max}], \\
        \big|\phi(t_i) -\phi(t_{i-1})\big| \leq \Delta \phi,~\phi(t_i) \in [0, 2\pi]. 
    \end{cases}
\end{equation}

Based on the C\&C signal, the kinematic model of the UAV can be given as
\begin{equation} \label{kine}
    \begin{cases}
        p_x(t_{i+1}) =  p_x(t_i) + V(t_{i-1})\cos(\phi(t_{i-1}))\tau_c(t_i) +\\~~~~~~~~~~~~~V(t_i)\cos(\phi(t_i))(\Delta t - \tau_c(t_i)) + \eta_x(t_i) , \\
        p_y(t_{i+1}) =  p_y(t_i) + V(t_{i-1})\sin(\phi(t_{i-1}))\tau_c(t_i)+ \\~~~~~~~~~~~~~V(t_i)\sin(\phi(t_i))(\Delta t - \tau_c(t_i)) + \eta_y(t_i),
    \end{cases}
\end{equation}
where $\eta_x(t_i) ,\eta_y(t_i) \sim \mathcal{N}(0, \sigma_\eta^2)$ are Gaussian noise due to environmental disturbance or system imperfection.

\subsection{Beamforming Model}
The beampattern synthesis approach in \cite{BF} is adopted to design the beamforming vector $\bm{f}(t_i)$ for sensing and communication. Specifically, we construct an uniform angular grid covering $[-\frac{\pi}{2}, \frac{\pi}{2}]$, which consists of $\mathrm{N}_{\vartheta}$ discrete grid points denoted by $\{\vartheta_u\}_{u=1}^{\mathrm{N}_{\vartheta}}$.
Given the detected AoA $\hat{\theta}(t_i)$ and variance $\sigma^2(\hat{{\theta}}(t_i))$, we define the scope of the direction from the BS towards the UAV at time $t_i$ as $\bm{\vartheta}_{\text{CI}}(t_i) = [\vartheta_{\min}(t_i), \vartheta_{\max}(t_i)]$, where the boundary angles are given as
\begin{equation}
    \begin{cases}
        \vartheta_{\min}(t_i) = \hat{\theta}(t_i) - \mathrm{B}\sigma(\hat{\theta}(t_i)), \\
        \vartheta_{\max}(t_i) = \hat{\theta}(t_i) + \mathrm{B}\sigma(\hat{\theta}(t_i)),
    \end{cases}
\end{equation}
in which $\mathrm{B}$ is the confidence level factor. Let $\bm{b}_D(t_i) \in \mathbb{C}^{\mathrm{N}_{\vartheta}\times 1}$ denote the desired beam pattern over the angular grid. Its entries are defined as
\begin{equation}
    [\bm{b}_D(t_i)]_u \triangleq \begin{cases}
        \mathrm{K}, \quad \text{if } \vartheta_u \in  \bm{\vartheta}_{\mathrm{CI}}(t_i), \\
        0, \quad \text{otherwise}.
    \end{cases}
\end{equation}

The beam pattern synthesis problem can be formulated as 
\begin{equation}
    \min_{\bm{f}(t_i)} \big\Vert \mathbf{b}(t_i) - \bm{A}^\top_{\text{tx}}\bm{f}(t_i)\big\Vert^2
\end{equation}
where $\bm{A}_{\text{tx}} = [\bm a_{\text{tx}}(\vartheta_1),...,\bm a_{\text{tx}}(\vartheta_{\mathrm{N}_{\vartheta}}) ] \in \mathbb{C}^{\mathrm{K}\times \mathrm{N}_{\vartheta}}$ denotes the transmit steering matrix evaluated at the grid points. This least squares problem has a closed-form solution
\begin{equation}
    \bm{f}(t_i) = (\bm A_{\text{tx}}^*\bm A_{\text{tx}}^\top)^{-1}\bm A_{\text{tx}}^*\bm{b}_D(t_i).
\end{equation}

\subsection{Obstacle Kinematic Model and Detection Model}
Dynamic obstacles are randomly distributed along the flight route of the UAV. The kinematic model of any obstacle can be denoted as \cite{obstacle_dynamics}
\begin{equation}
    \begin{cases}
        p_{ox}(t_{i+1}) =  p_{ox}(t_i) + V_{ox}(t_{i})\Delta t , \\
        p_{oy}(t_{i+1}) =  p_{oy}(t_i) + V_{oy}(t_{i})\Delta t,
    \end{cases}
\end{equation}
where $V_{ox}(t_i)$, $V_{oy}(t_i)$ are the obstacle velocities along the $x$- and $y$-axes, respectively, which are not known in advance. Due to the uncertainty of obstacle locations and movements, the cooperative BSs need to detect obstacle positions at each time slot. The detailed detection process is not the focus of this work. Instead, we leverage the detected obstacle positions and their uncertainties at each time slot. For safety, the cooperative BSs scan the surrounding environment within a circular region centered at the UAV \cite{scan_range}. Accordingly, the set of obstacles detected at time $t_i$ can be expressed as
\begin{equation}
    \hat{\mathcal{O}}(t_i) = \big\{o~\big|~\Vert \bm{p}_o(t_i) -\hat{\bm{p}}(t_i)\Vert \le \mathrm{R}_{\text{scan}} \big\},
\end{equation}
in which $\mathrm{R}_{\text{scan}}$ is the scanning radius. Let $\epsilon_{ox}(t_i)$ and $\epsilon_{oy}(t_i)$ denote the detection errors of obstacles along the $x$- and $y$-axes at time $t_i$. These errors are modeled as Gaussian random variables
\begin{equation}
    \epsilon_{ox}(t_i), \; \epsilon_{oy}(t_i) \sim \mathcal{N}(0, \sigma_o^2),
\end{equation}
where $\sigma_o^2$ is the variance of the obstacle detection noise.

\vspace{-0.2 cm}
\subsection{Process Overview and Problem Formulation}
The overall process of the obstacle avoidance task is described as follows. At the beginning of the task, the BS is assumed to have prior knowledge of the UAV’s real initial position $\bm{p}(t_0)$ and its real destination position $\bm{p}_{\text{dst}}$. During subsequent time slots, the BS determines whether to transmit sensing and/or C\&C signals based on the detected UAV position $\hat{\bm{p}}(t_i)$ and its associated detection variances $[\sigma^2_x(t_i),~\sigma^2_y(t_i)]^\top$, as well as the detected positions of dynamic obstacles $\hat{\bm{p}}_o(t_i)$ and their corresponding variances $[\sigma^2_{o},~\sigma^2_{o}]^\top$, where $o \in \hat{\mathcal{O}}(t_i)$. For ease of understanding, suppose the BS transmits a sensing signal at the $t_i$-th time slot, from which the UAV’s detected position $\hat{\bm{p}}(t_i)$ and the detection variance $[\sigma^2_x(t_i),~\sigma^2_y(t_i)]^\top$ can be obtained. Let us assume there are dynamic obstacles around the UAV at $t_i$, and their detected positions and variances are $\hat{\bm{p}}_o(t_i)$ and $[\sigma^2_{o},~\sigma^2_{o}]^\top$, respectively. Based on the detected positions and variances of the UAV and obstacles, the BS decides whether to transmit a C\&C signal and a new sensing signal at the next time slot $t_{i+1}$. This process is iteratively executed until the UAV successfully completes its task.

Our objectives are twofold: 1) to enable the UAV to complete the task as quickly as possible, and 2) to reduce the number of transmitted sensing and C\&C signals. Therefore, the objective function is denoted as 

\begin{align}
    \mathcal{P}1.1:&~\mathop{\rm min}_{\substack{\delta_c(t_i),\delta_s(t_i) \\V(t_i),\phi(t_i) }} ~ t_I   \label{eq1}    \\
       \textrm{s. t.}&~\big\Vert\bm{p}(t_i)  - \bm{p}_o(t_i)\big\Vert > \mathrm{D}_{\text{safe}}, \quad \forall t_i,~o, \tag{\ref{eq1}{a}}   \label{eq1a} \\
      &~\big\Vert\bm{p}(t_I) - \bm{p}_{\text{dst}}\big\Vert \leq \mathrm{D}_{\text{thr}}, \tag{\ref{eq1}{b}}   \label{eq1b} \\
    \mathcal{P}1.2:&~\mathop{\rm min}_{\substack{\delta_c(t_i),\delta_s(t_i) \\V(t_i),\phi(t_i) }} \sum_{i=0}^{I} \Big(\delta_c(t_i) + \delta_s(t_i) \Big)  \label{eq2}  \\
       \textrm{s. t.}&~\eqref{eq1a},~\eqref{eq1b}  \notag,
\end{align}
where $t_I$ in \eqref{eq1} denotes the time slot index for the UAV to reach its destination, $\sum_{i=0}^{I} (\delta_c(t_i) + \delta_s(t_i))$ in \eqref{eq2} represents the total number of transmitted sensing and C\&C signals within $t_I$, $\mathrm{D}_{\text{safe}}$ is the safety distance for the UAV to avoid collision, $\mathrm{D}_{\text{thr}}$ is the threshold distance for UAV to reach its destination.

\vspace{-0.1 cm}
\section{Traditional Framework}  \label{trad}
In traditional research, obstacle avoidance and ISAC signal transmission have been studied as separate disciplines. Obstacle avoidance is primarily explored in the field of mobile robotics, whereas ISAC signal transmission is investigated in wireless communications. In existing ISAC research, sensing and communication signals are typically directed towards different targets, and both signals are transmitted continuously \cite{continuous}. 
Robotic obstacle avoidance research commonly employs the dynamic window approach (DWA) \cite{DWA} in highly dynamic environments, which can guide the robot to avoid dynamic obstacles while moving towards its destination. The basic principle of DWA is to sample the robot’s linear and angular velocities within a feasible range constrained by its kinematics. Each velocity pair is evaluated using a cost function that accounts for distances to obstacles and the destination, and the optimal pair with the minimum cost is then selected. However, the DWA relies on the Euclidean norm to calculate distances, which leads to potential collisions in the presence of position detection errors. A typical solution is to introduce an inflation radius around the obstacle to account for such uncertainties \cite{inflation}. As a result, the distance is calculated using both the Euclidean norm and the inflation radius. 

\begin{figure}[t] 
    \centering
    \includegraphics[width=0.3\textwidth]{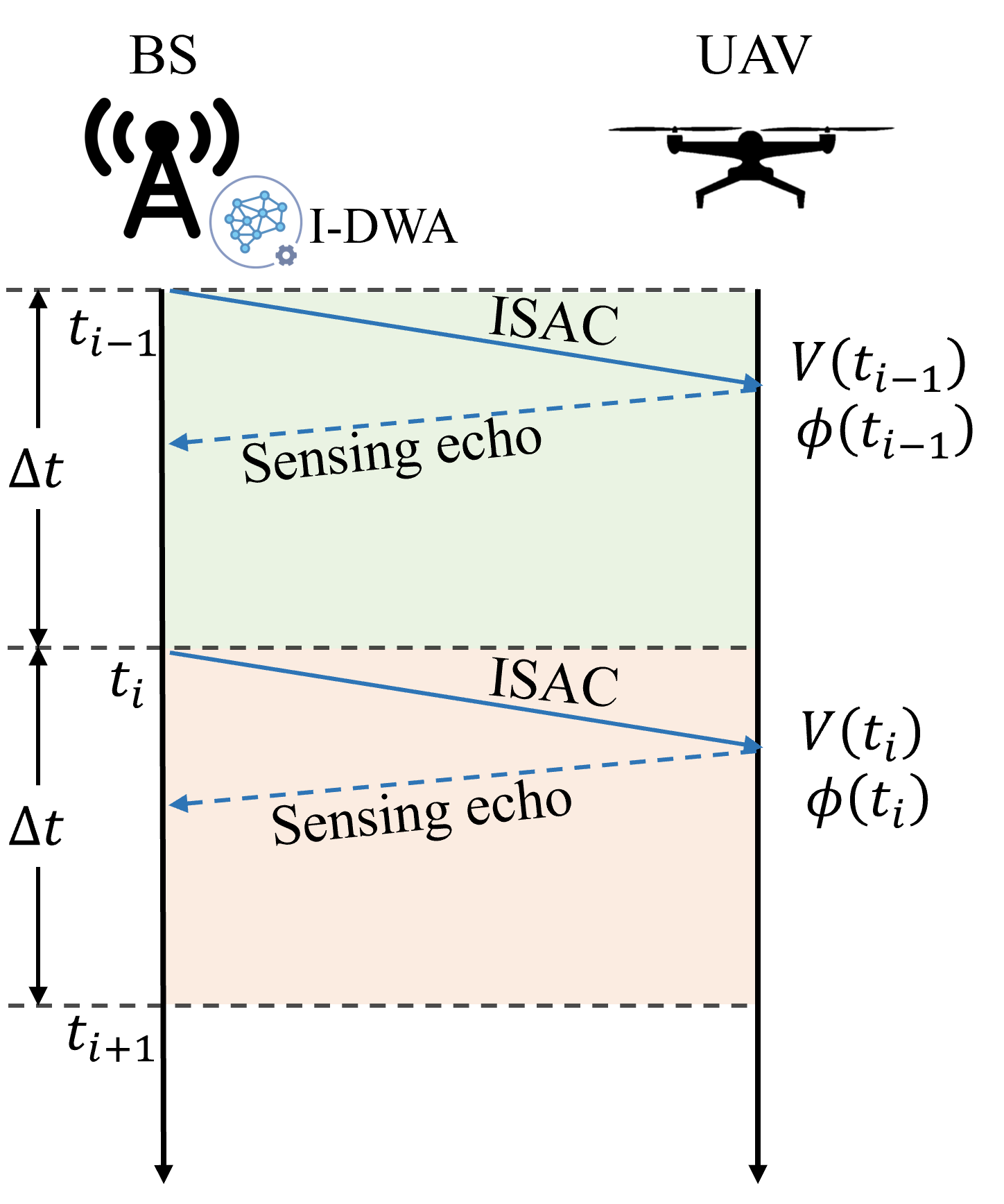}
    \caption{Traditional signal transmission framework.}
    \label{trad_framework}
\end{figure}

\figref{trad_framework} illustrates an example of the signal transmission under the traditional obstacle avoidance method. Specifically, at the beginning of the $t_i$-th time slot, the BS transmits an ISAC signal to the UAV, in which the C\&C values are generated by the inflation based-DWA (I-DWA) using the detection positions of UAV and obstacles together with their variances at the $t_{i-1}$-th time slot. The sensing signal is reflected by the UAV and received by the BS. The C\&C signal is received by the UAV, and then the UAV updates its motion state and moves at speed $V(t_i)$ and heading angle $\phi(t_i)$. 

It is worth noting that, to ensure the safe operation of robotic systems, C\&C signals are typically transmitted at very high frequencies \cite{wu}. Since C\&C signals strongly depend on the sensed positions, the ISAC signals in our considered task must also be transmitted at such high frequencies, resulting in a large number of redundant C\&C transmissions. For example, when there are no obstacles around the UAV, the BS may not need to transmit new sensing or C\&C signals, and the UAV can continue executing the previously received C\&C signal for several time slots despite random environmental influence. This motivates the development of a more efficient strategy that transmits sensing and/or C\&C signals only when they are significantly beneficial to the goal of the task.

\section{GOSC Framework}  \label{proposed}
In this section, we present our proposed GOSC framework for ISAC-enabled robotic obstacle avoidance task. The framework comprises three main components that constitute a closed loop for sensing, C\&C generation, and transmission for each time slot: a Kalman filter (KF) to mitigate the dependence of UAV position estimation on sensing signal transmissions, a Mahalanobis distance-based dynamic window approach (MD-DWA) to generate accurate C\&C signals under uncertainty, and an effectiveness-aware deep Q-network (E-DQN) to ensure that sensing and C\&C signals are transmitted only when they provide sufficient benefit to the task. These three components are introduced in detail in the subsequent subsections, followed by an overall description of the signal transmission process within the GOSC framework. For ease of distinction, the UAV's position detected from the sensing signal is denoted by $\hat{\bm p}$, the predicted position by the KF is denoted as $\mathring{\bm p}$, and the final estimation position output from the KF is represented as $\widetilde{\bm p}$.

\subsection{Kalman Filter}
We propose a KF that consists of two main functions: 1) predicting the position of the UAV at every time slot, and 2) reducing uncertainty of the sensing detected position. The first function provides a reference position for the UAV even when no signals are transmitted, such that the BS does not need to transmit sensing signals all the time. While the second function refines the detected position to achieve higher accuracy. To begin, we introduce the process transition model of the KF
\begin{equation} \label{pre_ekf}
    \acute{\bm{p}}(t_i) = \acute{\bm{p}}(t_{i-1}) + \bm{D}\bm{v}(t_{i-1}) + \bm{\gamma}(t_i),
\end{equation}
where $\acute{\bm{p}}(t_i)$ is the true position of UAV at $t_i$ estimated by the KF,  
$\bm{v}(t_{i-1})=[V(t_{i-1})\cos(\phi(t_{i-1})),V(t_{i-1})\sin(\phi(t_{i-1}))]^\top$ represents the velocity along 2D Cartesian coordinate axes at the $t_{i-1}$-th time slot,  $\bm{D}\in \mathbb{R}^{2\times2}$ is a matrix indicating the duration of speed $\bm{v}(t_{i-1})$, and $\bm{\gamma}=[\eta_x(t_i), \eta_y(t_i)]^\top$ is the process noise. Due to the LoS-dominant channel and the relatively short transmission distance, the transmission delay $\tau_c(t_{i-1})$ of the C\&C signal is small compared to the time slot duration $\Delta t$. Moreover, since C\&C signals are not transmitted at every time slot and the BS cannot precisely measure $\tau_c(t_{i-1})$, the impact of such delay is incorporated into the process noise. Accordingly, we set $\bm{D}=\bm{I}_2\Delta t$.
By \eqref{pre_ekf}, the KF can predict the UAV's position and its covariance matrix at every time slot. Specifically, the predicted position of the UAV at $t_{i}$ can be denoted as 
\begin{equation}  \label{ekf_predict}
    \mathring{\bm{p}}(t_{i}) = \mathring{\bm{p}}(t_{i-1}) + \bm{D}\bm{v}(t_{i-1}).
\end{equation}

The covariance matrix $\mathring{\bm{\Gamma}}(t_{i})\in \mathbb{R}^{2\times2}$ of the predicted position $\mathring{\bm{p}}(t_{i})$ is represented as 
\begin{equation}  \label{cm_pre}
    \mathring{\bm{\Gamma}}(t_{i}) = \mathring{\bm{\Gamma}}(t_{i-1}) + \bm{Q},
\end{equation}
where $\mathring{\bm{\Gamma}}(t_0)=\bm{0}_{\mathrm{2}}$ is the initial position of the UAV that is known at the BS, $\bm{Q}=\bm{I}_2 \sigma^2_\eta$ is the covariance matrix of process noise. It can be found that although we can obtain a prediction position of the UAV at every time slot, the accumulated error increases if no sensing signals are transmitted over time, which in turn degrades the reliability of obstacle avoidance and may lead to collision. If the BS decides to transmit a sensing signal at the $t_{i}$-th time slot, according to equation \eqref{est_real}, the sensing position is represented as $\hat{\bm{p}}(t_{i})= \bm{p}(t_{i}) + [\epsilon_x(t_{i}), \epsilon_y(t_{i})]^\top$. By combining the prediction and sensing uncertainties, the Kalman gain is obtained as
\begin{equation}
    \mathbf{G}(t_{i}) = \mathring{\bm{\Gamma}}^\top(t_{i})\Big[\mathring{\bm{\Gamma}}(t_{i}) + \hat{\bm{\Gamma}}(t_{i}) \Big]^{-1},
\end{equation}
where $\hat{\bm{\Gamma}}(t_{i}) \in \mathbb{R}^{2\times 2}$ is the covariance matrix of $\hat{\bm{p}}(t_{i})$ and is denoted as
\begin{equation}
\begin{aligned}
\hat{\bm{\Gamma}}(t_{i}) = \begin{bmatrix}
    \sigma^2_x(t_i), & 0 \\
    0, & \sigma^2_y(t_i)
\end{bmatrix}.
\end{aligned}
\end{equation}

The Kalman gain $\mathbf{G}(t_{i})$ serves as a weighting factor to balance the prediction position $\mathring{\bm{p}}(t_{i})$ and the sensing estimation position $\hat{\bm{p}}(t_{i})$. Such that the KF can provide a more accurate position estimation and reduce the overall estimation uncertainty. The refined estimation position and its covariance matrix are given by
\begin{equation} \label{refine_pos}
\begin{cases}
    \widetilde{\bm{p}}(t_{i})=\mathring{\bm{p}}(t_{i}) + \bm{G}(t_{i})\big(\hat{\bm{p}}(t_{i})-\mathring{\bm{p}}(t_{i})\big), \\
    \widetilde{\bm{\Gamma}}(t_{i})=\big(\bm{I}_2-\bm{G}(t_{i})\big)\mathring{\bm{\Gamma}}(t_{i}).
\end{cases}
\end{equation}

To enhance clarity, the workflow of the proposed KF is presented as follows: If no sensing signal is transmitted at $t_i$, the KF outputs $\widetilde{\bm p}(t_i)=\mathring{\bm{p}}(t_i)$ along with the covariance matrix $\widetilde{\bm{\Gamma}}(t_i) = \mathring{\bm{\Gamma}}(t_i)$, which together serves as the estimated information of the UAV. If a sensing signal is transmitted at $t_i$, the KF provides the refined position $\widetilde{\bm{p}}(t_i)$ and the covariance matrix $\widetilde{\bm{\Gamma}}(t_i)$ according to \eqref{refine_pos}. At the same time, the prediction covariance matrix $\mathring{\bm{\Gamma}}(t_i)$ is updated as $\widetilde{\bm{\Gamma}}(t_i)$. If a C\&C signal is transmitted at $t_{i+1}$, the velocity vector $\bm{v}(t_{i-1})$ in \eqref{ekf_predict} is updated according to equation \eqref{speed}.

\subsection{Mahalanobis Distance-based DWA}
Based on $\widetilde{\bm p}(t_i)$ and ${\widetilde{\bm\Gamma}}(t_i)$, as well as the detected obstacle positions $\hat{\bm{p}}_o(t_i)$ and their corresponding variances $[\sigma_o^2,~\sigma_o^2]^\top$, the BS generates C\&C signals to guide the UAV avoiding obstacles while moving towards its destination. Traditional inflation-based C\&C methods are overly conservative: they ignore the directional characteristics of uncertainty and often overestimate obstacle regions, which reduces navigation efficiency.
To overcome these limitations, we design a MD-DWA method to generate C\&C signals. Specifically, the Mahalanobis distance \cite{mahala} between the UAV and the detected obstacle $o\in\hat{\mathcal{O}}(t_i)$ is defined as 
\begin{equation}  \label{MD}
\begin{aligned}
    d_o^M&= \|\widetilde{\bm p} -\bm p_o \|_M \\
    &=\sqrt{\big(\widetilde{\bm{p}} -\bm{p}_o\big)^\top\big(\widetilde{\bm{\Gamma}}+\bm{\Phi}_o\big)^{-1}\big(\widetilde{\bm{p}}-\bm{p}_o\big)}, 
\end{aligned}
\end{equation}
where the time slot index is $t_i$ and is omitted due to space limitation, $\bm{\Phi}_o=\bm{I}_2 \sigma^2_o$ is the covariance matrix of the $o$-th detected obstacle. The detailed generation process of C\&C signals based on the MD-DWA method is described as follows. Suppose the BS transmits a sensing signal at $t_i$, after receiving the sensing echo signal, the BS generates a feasible movement velocity and forward angle set. According to equation \eqref{cc_cons}, the feasible set can be denoted as
\begin{equation} \label{omega}
\begin{aligned}
    \bm{\Omega}(t_i)=&\bigg\{\big(\breve{V}(t_i), \breve{\phi}(t_i)\big) \Big|
 \max\{0, V(t_{i-1}) - \Delta \mathrm{V}\} \leq \breve{V}(t_i) 
   \leq \\
&\min \{V(t_{i-1}) + \Delta \mathrm{V},\mathrm{V_{max}}\},\max\{0, \phi(t_{i-1}) - \Delta \phi\}  \\
&\leq \breve{\phi}(t_i) \leq \min \{\phi(t_{i-1}) + \Delta \phi, 2\pi \}\bigg\}.
\end{aligned}
\end{equation}

For each candidate pair $(\breve{V}(t_i), \breve{\phi}(t_i)) \in \bm{\Omega}(t_i)$, a forward prediction trajectory $\bm{\Xi}(t_i)$ is generated over the next $\mathrm{I_p}$ time slots using the KF prediction model \eqref{ekf_predict}, which can be expressed as the discrete sequence
\begin{equation} \label{pre_traj1}
    \bm{\Xi}(t_i) = \Big\{\big[\breve{p}_x(t_{i+b}),~\breve{p}_y(t_{i+b})\big]^\top:~b=1,2,...,\mathrm{I_p}\Big\},
\end{equation}
in which
\begin{equation}
     \begin{bmatrix}
         \breve{p}_x(t_{i+b}) \\
         \breve{p}_y(t_{i+b}) \\
     \end{bmatrix}=\begin{bmatrix}
         \widetilde{p}_x(t_i) + \breve{V}(t_i)\cos(\breve\phi(t_i))b\Delta t \\
         \widetilde{p}_y(t_i) + \breve{V}(t_i)\sin(\breve\phi(t_i))b\Delta t
     \end{bmatrix},
\end{equation}

To evaluate collision risk along the predicted trajectory, the minimum Mahalanobis distance to the detected obstacles is calculated as
\begin{equation}
    d^M_{\min}\big(\breve{V}(t_i), \breve{\phi}(t_i)\big) = \min_{b}\min_{o}  ~d_o^M(t_{i+b}) ,
\end{equation}
in which $\bm{p}_o(t_{i+b})=\bm{p}_o(t_i)$ is constant over the horizon since obstacle motions are unknown and need to be detected by cooperative BSs,  $\bm{\Phi}_o(t_{i+b})=\bm{\Phi}_o(t_{i})+b\Delta t\bm{I}_2$. Let $b^\star$ and $o^\star$ denote the prediction step and obstacle yielding this minimum Mahalanobis distance, respectively.

\begin{lemma}
    To avoid collision with obstacles, the minimum Mahalanobis distance must satisfy
    \begin{equation}  \label{md_cons}
        d^M_{\min}\big(\breve{V}(t_i), \breve{\phi}(t_i)\big) \geq \sqrt{\chi^2_{2,0.99}}  +  \frac{\mathrm{D}_{\mathrm{safe}}}{\sqrt{\lambda_{\min}(\bm{\Sigma}(t_i))}},
    \end{equation}
    where $\chi^2_{2,0.99}$ is the the $99\%$ confidence threshold of a chi-squared ($\chi^2$) distribution with two degrees of freedom, $\bm{\Sigma}(t_i)= \bm{\Phi}_{o^\star}(t_i) + \mathring{\bm{\Gamma}}(t_{i+b^\star})$ is the covariance matrix of $d^M_{\min}(\breve{V}(t_i), \breve{\phi}(t_i))$, $\lambda_{\min}(\bm{\Sigma}(t_i))$ is the minimum eigenvalue of $\bm{\Sigma}(t_i)$.
\end{lemma}
\begin{proof}
    Please see Appendix B.
\end{proof}

The feasible candidate set is thus refined to $\bm{\Omega}’(t_i)$, which only includes the candidate pair $(\breve{V}(t_i), \breve{\phi}(t_i))$ satisfying \eqref{md_cons}. 
For each $(\breve{V}(t_i), \breve{\phi}(t_i)) \in \bm{\Omega}'(t_i)$, the Mahalanobis distance between the predicted trajectory endpoint $\breve{\bm{p}}(t_{i+\mathrm{I_p}})$ and the destination $\bm{p}_{\text{dst}}$ is computed as 
\begin{equation} \label{m_des}
\begin{aligned}
        d^M_{\mathrm{des}}\big( &\breve{V}(t_i), \breve{\phi}(t_i)\big)=  \\
&\sqrt{\big(\breve{\bm{p}}(t_{i+\mathrm{I_p}}) -\bm{p}_{\text{dst}}\big)^\top\big(\breve{\bm{\Gamma}}(t_{i+\mathrm{I_p}})\big)^{-1}\big(\breve{\bm{p}}(t_{i+\mathrm{I_p}}) -\bm{p}_{\text{dst}}\big)}.
\end{aligned}
\end{equation}

To realize obstacle avoidance and destination approaching related to $\mathcal{P}1.1$, an evaluation function is defined as 
\begin{equation}  \label{evalu}
    E\big(\breve{V}(t_i), \breve{\phi}(t_i)\big)= \frac{d^M_{\text{dst}}\big( \breve{V}(t_i), \breve{\phi}(t_i)\big)}{d^{M^\star}_{\text{dst}}} + \frac{d^{M^\star}_{\min}}{d^M_{\min}\big( \breve{V}(t_i), \breve{\phi}(t_i)\big)}, 
\end{equation}
where $d^{M^\star}_{\text{dst}}$ and $d^{M^\star}_{\min}$ are the maximum and the minimum values of $d^M_{\text{dst}}\big( \breve{V}(t_i), \breve{\phi}(t_i)\big)$ and $d^M_{\min}\big( \breve{V}(t_i), \breve{\phi}(t_i)\big)$, respectively. The optimal control pair is then selected as 
\begin{equation}  \label{cc_cont}
    \big(V(t_i), \phi(t_i)\big) = \mathop{\arg \min}_{(\breve{V}(t_i), \breve{\phi}(t_i))\in \bm{\Omega}'(t_i)}E\big(\breve{V}(t_i), \breve{\phi}(t_i)\big)
\end{equation}

The overall procedure of the MD-DWA based C\&C signal generation is summarized in \textbf{\algref{DWA}}. 

\begin{algorithm}[h]
    \caption{The MD-DWA based C\&C Signal Generation}
    \label{DWA}
    \begin{algorithmic}[1]
    \REQUIRE
    Estimated UAV and obstacles positions with covariance matrices.
    \ENSURE
    The C\&C velocity $V(t_i)$ and heading angle $\phi(t_i)$.
    \WHILE{the task is not completed}
    \IF{a sensing signal is transmitted at $ t_{i-1}$}
    \STATE Generate a feasible set $\bm{\Omega}(t_i)$ via \eqref{omega}.
    \STATE Generate prediction trajectories based on each candidate pair in $\bm{\Omega}(t_i)$ according to \eqref{pre_traj1}.
    \IF{no obstacles are detected}
    \STATE Evaluate candidates in $\bm{\Omega}(t_i)$ using \eqref{m_des}.
    \ELSE
    \STATE Refine the feasible set as $\bm{\Omega}'(t_i)$ via \eqref{md_cons}.
    \STATE Calculate the evaluation function according to \eqref{evalu}.
    \ENDIF
    \ENDIF
    \STATE Select the optimal $V(t_i)$ and $\phi(t_i)$ based on \eqref{cc_cont}.
    \ENDWHILE
    \end{algorithmic}
\end{algorithm}

\textit{Remark 2: Computational Complexity Analysis of Algorithm 1.} The construction of the feasible control set $\bm{\Omega}(t_i)$ at step 3 requires discretizing the admissible velocity and heading angle spaces. Let $N_V$ and $N_{\phi}$ denote the number of discretization levels for velocity and heading angle, respectively. The complexity of generating $\bm{\Omega}(t_i)$ is therefore $\mathcal{O}(N_V N_{\phi})$. At step 4, for each candidate control pair in $\bm{\Omega}(t_i)$, a prediction trajectory over a horizon of length $\mathrm{I_p}$ is generated according to \eqref{pre_traj1}. This step incurs a complexity of $\mathcal{O}(N_V N_{\phi} \mathrm{I_p})$. The evaluation of all candidate pairs (Steps 5--10) involves computing the corresponding objective values, which has a complexity of $\mathcal{O}(N_V N_{\phi})$. Finally, the `$\arg\min$' operation in \eqref{cc_cont} also requires $\mathcal{O}(N_V N_{\phi})$. Combining the above steps, the overall computational complexity of \textbf{\algref{DWA}} is $\mathcal{O}\big(N_V N_{\phi} \mathrm{I_p}\big)$, where the dominant term is linear in the prediction horizon $\mathrm{I_p}$. This indicates that the algorithm scales linearly with both the discretization granularity and the prediction horizon, and is therefore suitable for real-time implementation.

\subsection{Effectiveness-Aware DQN for Sensing and C\&C Signal Transmission}
Based on the KF and MD-DWA, we develop an E-DQN to control the transmission of sensing and C\&C signals. To align with the objective functions in \eqref{eq1} and \eqref{eq2}, the state, action, and reward in E-DQN are carefully designed. In particular, to meet the objectives in $\mathcal{P}1.1$/$\mathcal{P}1.2$—namely, enabling the UAV to reach its destination as quickly as possible without collision while minimizing the total number of transmitted signals—the MDP state is designed to capture goal-oriented information. Specifically, the observation state of E-DQN at the $t_i$-th time slot is defined as
\begin{equation}
\mathcal{S}(t_i) \triangleq \Big\{d_{\text{dst}}^M(t_i),~d_{\text{obs}}^M(t_i),~\det\big(\bm{{\widetilde{\bm\Gamma}}}(t_i)\big),~t_i,~N_{\delta}(t_i)\Big\},
\end{equation}
where $d_{\text{dst}}^M(t_i)=\|\widetilde{\bm p}(t_i)- \bm p_{\text{dst}} \|_M$ is the Mahalanobis distance to the destination, $d_{\text{obs}}^M(t_i)=\min \{d_o^M | o \in \hat{\mathcal{O}}(t_i) \}$ is the Mahalanobis distance to the detected obstacles,
 $\det(\bm{\widetilde{\bm\Gamma}}(t_i))$ is the determinant of $\bm{\widetilde{\bm\Gamma}}(t_i)$, $N_{\delta}(t_i)=\sum_i(\delta_c(t_i)+\delta_s(t_i) )$ is the accumulated number of transmitted signals up to $t_i$. To solve the transmission decision variables $\delta_c(t_i)$ and $\delta_s(t_i)$ in $\mathcal{P}1.1$/$\mathcal{P}1.2$, the action of DQN at $t_i$ is defined as
\begin{equation}
    \mathcal{A}(t_i) \triangleq \Big\{0,~1,~2\Big\},
\end{equation}
where $0,~1,~2$ represents stay silent, transmit sensing signal at $t_i$ and do not transmit communication signal at $t_{i+1}$, transmit sensing signal at $t_i$ and transmit communication signal at $t_{i+1}$, respectively. The reward of E-DQN at $t_i$ is defined as
\begin{equation}
    R(t_i) \triangleq \text{VoI}(t_i) - \text{Cost}(t_i) - \Psi_{\text{step}}(t_i) -\Psi_{\text{col}}(t_i),
\end{equation}
where each term is specified below:

\noindent\textbf{(a)} $\text{VoI}(t_i)$ contains the VoI of sensing and communication signals, which is defined as 
\begin{equation}
    \text{VoI}(t_i) \triangleq \text{VoI}_s(t_i) + \text{VoI}_c(t_i),
\end{equation}
in which the VoI of the sensing signal $\text{VoI}_s(t_i)$ is quantified by the reduction in estimation error entropy
\begin{equation}
\begin{aligned}
    \text{VoI}_s(t_i) 
    &\triangleq \mathbb{H}(\bm{\widetilde{\bm\Gamma}}(t_{i-1})) - \mathbb{H}(\bm{\widetilde{\bm\Gamma}}(t_{i})) \\
    &= \frac{1}{2}\ln\bigg(\frac{\det\big(\bm{\widetilde{\bm\Gamma}}(t_{i-1})\big)}{\det\big(\bm{\widetilde{\bm\Gamma}}(t_i)\big)} \bigg),
\end{aligned}
\end{equation}
where $\mathbb{H}(\bm{\widetilde{\bm\Gamma}}(t_i))$ denotes the entropy of the covariance matrix $\bm{\widetilde{\bm\Gamma}}(t_i)$. The VoI of the C\&C signal $\text{VoI}_c(t_i)$ can be measured by navigation improvement of the UAV
\begin{equation}
    \text{VoI}_{c}(t_i) \triangleq \Delta||\mathring{\bm{p}}(t_{i+1}) - \bm p_{\text{dst}}|| + \Delta1_{\text{col}}(t_i),
\end{equation}
where $\Delta||\mathring{\bm{p}}(t_{i+1}) - \bm p_{\text{dst}}||$ is the difference in UAV–destination distance with and without transmitting the C\&C signal, $\mathring{\bm{p}}(t_{i+1})$ is the KF-predicted UAV position at time slot $t_{i+1}$ based on the real UAV position $\bm{p}(t_i)$. 
It is worth noting that the proposed E-DQN operates in two different phases, i.e., training phase and deployment phase. Real UAV position $\bm{p}(t_i)$ is observable during the training phase. During the deployment phase, the reward function is no longer calculated or required, and the trained policy operates solely on observation state $\mathcal{S}(t_i)$ and outputs an action, thereby mitigating the training-deployment gap. $\Delta 1_{\text{col}}(t_i)=1$ if transmitting the C\&C signal at $t_i$ can help UAV avoid obstacle while no transmission causes collision, and $0$ otherwise. Although a large collision penalty $\Psi_{\text{col}}(t_i)$ is introduced to strongly discourage unsafe behaviors, relying solely on such a sparse and terminal penalty may lead to slow convergence and unstable learning. Thus, we introduce $\Delta1_{\text{col}}(t_i)$ as a step-wise shaping reward, which provides intermediate and informative feedback, guiding the agent to recognize the importance of C\&C transmission in safety-critical situations. As a result, the agent can learn safer navigation strategies more efficiently.

\noindent\textbf{(b)}  
$\text{Cost}(t_i)$ is the cost for transmitting sensing and C\&C signals. Let us denote $A(t_i) \in \mathcal{A}(t_i)$, and $\text{Cost}(t_i)$ can be defined as
\begin{equation}
    \text{Cost}(t_i) \triangleq
    \begin{cases}
        0.5, & \text{if } A(t_i)  = 1, \\
        1,   & \text{if } A(t_i) = 2, \\
        0,   & \text{otherwise}.
    \end{cases}
\end{equation}

\noindent\textbf{(c)} $\Psi_{\text{step}}(t_i)$ is the step penalty introduced to encourage E-DQN to complete the task as quickly as possible. Its value increases gradually with $t_i$, thereby penalizing longer time. Specifically, it is defined as
\begin{equation}
    \Psi_{\text{step}}(t_i) \triangleq \frac{1}{1 + \mathrm{e}^{-0.1\cdot i}},
\end{equation}
which adopts a sigmoid form to ensure smooth growth and bounded normalization.

\noindent\textbf{(d)}  
$\Psi_{\text{col}}(t_i)$ is the collision penalty introduced to avoid collision, which is given as 
\begin{equation}
\Psi_{\text{col}}(t_i)\triangleq
    \begin{cases}
        \Psi, \quad \text{if } \big\Vert\bm{p}(t_i)  - \bm{p}_o(t_i)\big\Vert \le \mathrm{D}_{\text{safe}},  \\
        0, \quad \text{otherwise}.
    \end{cases}
\end{equation}

Note that the problems $\mathcal{P}1.1$ and $\mathcal{P}1.2$ aim to jointly minimize the task completion time and the total number of transmitted signals. These objectives are inherently sequential and decision-dependent, as the UAV state evolves over time and transmission decisions affect both future system states and accumulated costs. Therefore, based on above definitions, the problems can be naturally reformulated as a sequential decision-making problem, i.e., Markov decision process (MDP) problem:
\begin{equation}
\mathcal{P}2:~ \max_{\pi_{\omega}(A(t_i)|\mathcal{S}(t_i))} \; \mathbb{E}_{\pi_{\omega}} \Bigg[ \sum_{i=0}^I \gamma^i R(t_i) \Bigg]
\end{equation}
where $\pi_{\omega}(A(t_i)|\mathcal{S}(t_i))$ denotes the E-DQN policy parameterized by $\omega$, specifying which action $A(t_i)\in\mathcal{A}(t_i)$ is taken when observing state $\mathcal{S}(t_i)$, and $\gamma \in (0,1)$ is the discount factor that balances the importance of immediate and future rewards. A larger $\gamma$ encourages the agent to focus on long-term performance, which is essential in our problem since the objective is inherently long-term.
The formulation of expected cumulative reward $\sum_{i=0}^I \gamma^i R(t_i)$ in MDP serves as a surrogate for the original multi-objective problem: maximizing $\text{VoI}(t_i)-\text{Cost}(t_i)$ encourages transmissions only when they are critical to the task, while minimizing $\Psi_{\text{step}}(t_i)$ and $\Psi_{\text{col}}(t_i)$ promotes fast task completion and collision avoidance. Therefore, maximizing the expected cumulative reward in the MDP is aligned with solving $\mathcal{P}1.1$/$\mathcal{P}1.2$. From a theoretical perspective, the MDP formulation can be interpreted as a stochastic and model-free approximation of the original dynamic optimization problem, where the long-term objective is optimized via Bellman optimality. While an explicit closed-form equivalence is difficult to establish due to the coupled system dynamics and uncertainty, the proposed state/action/reward shaping ensures that the learned policy of DQN approximates the desired solutions to $\mathcal{P}1.1$/$\mathcal{P}1.2$.

Following the classical DQN structure \cite{DQN_nature}, two DNNs with identical architectures are maintained in E-DQN. One is called `eval network' with parameters $\omega$, used for learning and decision-making; the other is called `target network' with parameters $\omega^-$, updated periodically to stabilize training. To obtain the optimal parameters, the E-DQN framework is trained in the following way: At each time slot $t_i$, the eval network observes the current state $\mathcal{S}(t_i)$ and selects an action $A(t_i)$ according to $\pi_\omega$. To balance exploration and exploitation, the action is chosen via an $\varepsilon$-greedy strategy
\begin{equation} \label{DQN_action}
    A(t_i) \triangleq \begin{cases}
        \mathop{\arg\max}\limits_{A(t_i)\in\mathcal{A}(t_i)}~ Q(\mathcal{S}(t_i),A(t_i)|\omega), ~~\text{if } \varepsilon \leq \varepsilon_0, \\
        \text{a random action} , ~~ \text{if } \varepsilon > \varepsilon_0,
    \end{cases} 
\end{equation}
where $\varepsilon \in [0,1]$ is a random variable, $\varepsilon_0 \in [0,1]$ is the exploration probability threshold. $Q(\mathcal{S}(t_i),A(t_i)|\omega)$ is the state-action value function \cite{DQN} and defined as 
\begin{equation}
    Q\big(\mathcal{S}(t_i),A(t_i)|\omega\big) \triangleq \mathbb{E}_{\pi_\omega}\Bigg[ \sum_{\varsigma=0}^{I-i} \gamma^{\varsigma} R(t_{i+\varsigma}) \Big | \mathcal{S}(t_i), A(t_i) \Bigg].
\end{equation}

After executing action $A(t_i)$, the agent receives reward $R(t_i)$ and the environment transitions to next state $\mathcal S(t_{i+1})$. The agent stores the tuple $(\mathcal S (t_i),A(t_i),R(t_i),\mathcal S(t_{i+1}))$ as an experience into its memory buffer. When the memory buffer is full, the DQN is trained by randomly sampling $\mathrm{N}_0$ experiences
\begin{equation}
    (\mathcal S_n,A_n,R_n,\mathcal S'_n), \quad n=1,2,...,\mathrm{N}_0,
\end{equation}
where $\mathcal S_n$, $A_n$, $R_n$, and $\mathcal S'_n$ denote the current state, action, reward, and next state of the $n$-th sampled experience, respectively. Random sampling helps break the temporal correlation among experiences. Based on these random samples, E-DQN minimizes the mean squared error between the predicted and the target state-action values, which is calculated as
\begin{equation} \label{loss}
\begin{aligned}
    L(\omega) =& \frac{1}{\mathrm{N_0}}\sum_{n=1}^{\mathrm{N_0}}\Bigg[R_n
    + \gamma \max_{A_n'}Q\big(\mathcal S'_n,A_n'|\omega^-\big) \\
    &-Q(\mathcal S_n,A_n|\omega) \Bigg]^2.
\end{aligned}
\end{equation}

The eval network parameters $\omega$ are updated via gradient descent 
\begin{equation} \label{eval_parameter}
    \omega \leftarrow \omega + \mathrm{L_r}\nabla_{\omega}L(\omega),
\end{equation}
where $\nabla_{\omega}L(\omega)$ denotes the gradient of the loss function $L(\omega)$, and $\mathrm{L_r} \in (0, 1)$ is the learning rate of the eval network. To stabilize learning, the target network parameters are updated every $\mathrm{N_U}$ steps 
\begin{equation} \label{target_parameter}
    \omega^- \leftarrow \omega.
\end{equation}

The overall training procedure is summarized in \textbf{\algref{DQN}}.
\begin{algorithm}[h]
    \caption{E-DQN for Signal Transmission}
    \label{DQN}
    \begin{algorithmic}[1]
    \REQUIRE
    Observation $\mathcal{S}(t_i)$ at each time slot
    \ENSURE
    Action $A(t_i)$ at each time slot.
    \STATE Initialize eval network parameters $\omega$ and target network parameters $\omega^-\leftarrow\omega$. 
    \STATE Set memory buffer capacity $\mathrm{C_M}$. Initialize counter $c \gets 0$.
    \STATE Collect initial experiences until the memory buffer is full.
    \FOR {each episode}
    \WHILE {the task is not completed}
    \STATE Select action $A(t_i)$ based on state $\mathcal{S}(t_i)$ using the $\varepsilon$-greedy policy in \eqref{DQN_action}.
    \STATE Execute action $A(t_i)$. The environment transitions into next state $\mathcal S(t_{i+1})$ and the agent receives reward $R(t_i)$.
    \STATE Store experience tuple $(\mathcal{S}(t_i), A(t_i), R(t_i), \mathcal{S}(t_{i+1}))$ in the memory buffer at index $c \bmod \mathrm{C_M}$.
    \STATE Randomly sample $\mathrm{N_0}$ experiences from the memory buffer.
    \STATE Compute the loss $L(\omega)$ using \eqref{loss} and update eval-network parameters via \eqref{eval_parameter}.
    \IF{$c \bmod \mathrm{N_U} = 0$}
    \STATE Update target network parameters: $\omega^- \leftarrow \omega$.
    \ENDIF
    \STATE Update state $\mathcal{S}(t_i) \gets \mathcal{S}(t_{i+1})$ and counter $c \gets c+1$.
    \ENDWHILE
    \ENDFOR
    \end{algorithmic}
\end{algorithm}
   
\textit{Remark 3: Computational Complexity Analysis of Algorithm 2.} 
\textbf{\algref{DQN}} describes the training process of the E-DQN. The computational complexity is mainly determined by the size and structure of the underlying DNNs. Specifically, the eval network and target network share the same architecture. Suppose each network consists of $L$ layers, where the $l$-th layer contains $N_l$ neurons. The dominant complexity during training (steps 9--14) arises from forward and backward propagation, which can be expressed as $\mathcal{O}\big(\mathrm{N_0}\sum_{l=1}^{L-1} N_l N_{l+1}\big)$ per update step \cite{dqn_complexity}. This complexity accounts for both inference and gradient computation across all layers. After training converges, only steps 5--7 (1-greedy policy at step 5) and steps 14--15 are executed, without any parameter updates. The interaction with the environment requires only forward propagation through the eval network. Therefore, the computational complexity per decision step reduces to $\mathcal{O}\big(\sum_{l=1}^{L-1} N_l N_{l+1}\big)$. Since the adopted DNN is lightweight, this inference complexity is typically low, making the proposed E-DQN suitable for real-time UAV control applications.

\subsection{Sensing and C\&C Signal Transmission Process}
\begin{figure}[t] 
    \centering
    \includegraphics[width=0.3\textwidth]{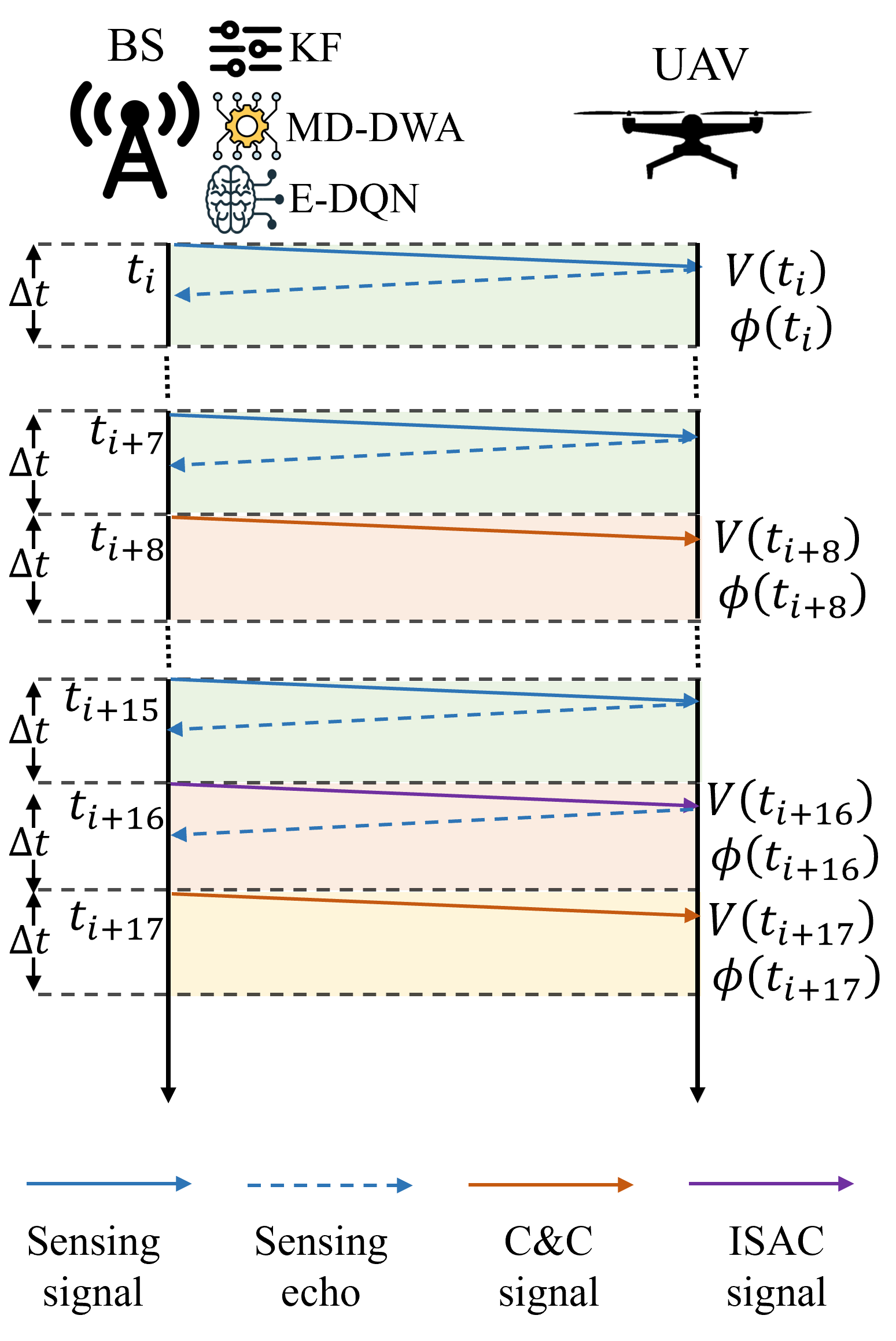}
    \caption{An example of sensing and C\&C signal transmission process under our proposed GOSC framework.}
    \label{pp_framework}
    \vspace{-0.3 cm}
\end{figure}

For better understanding, \figref{pp_framework} illustrates an example of the sensing and C\&C signal transmission process of our proposed GOSC framework. Specifically, at the beginning of the $t_i$-th time slot, E-DQN observes the environment state $\mathcal{S}(t_i)$ and decides to transmit only a sensing signal, i.e. $A(t_i)=1$. During the subsequent six time slots, the BS remains silent, i.e., $A(t_{i+(1\sim6)})=0$. At $t_{i+7}$, the E-DQN chooses to transmit a sensing signal at the current slot and a C\&C signal at the next slot $t_{i+8}$, i.e., $A(t_{i+7})=2$, followed by another seven idle slots. At $t_{i+15}$ and $t_{i+16}$, the E-DQN again selects $A(t)=2$. It can be observed that, unlike the traditional framework where ISAC signals are transmitted continuously, the BS in our GOSC framework selectively transmits sensing, C\&C, or ISAC signals. Consequently, the overall communication overhead can be significantly reduced.
 
\textit{Remark 4: Robustness of the Proposed GOSC Framework to the Sensing--Processing--Actuation Delay.} In practical implementations, sensing--processing--actuation delay introduces a temporal mismatch between the true system state and the state used for decision-making. In the proposed system, such delay is expected to remain moderate due to the relatively low computational complexity of the KF, MD-DWA, and E-DQN modules, as well as short sensing signal processing and communication latency. Moreover, the proposed GOSC framework is inherently robust to moderate delays. First, the KF explicitly models process noise, and its covariance propagation naturally captures the growth of uncertainty under delayed observations. Second, the MD-DWA incorporates uncertainty-aware safety constraints (e.g., Mahalanobis distance and the minimum safety distance $\mathrm{D_{safe}}$), which enhance robustness against state estimation errors induced by delays. Third, the VoI-based transmission mechanism adapts to increased uncertainty by triggering more frequent sensing and C\&C updates when necessary.

\section{Simulations and Analysis}  \label{simulation}
In this section, extensive simulations are conducted to evaluate the performance of our proposed GOSC framework. The key simulation parameters are summarized as follows.
The time slot interval is set to $\Delta t = 5$ ms.  
The UAV starts from the initial position $\bm{p}(t_0) = [0.1,~0.1]^\top$~m and flies toward the destination $\bm{p}_{\text{dst}} = [10,~10]^\top$~m, maintaining a fixed altitude of $\mathrm{H} = 10$~m.  
The radar cross section of the UAV is $\sigma_{\text{RCS}} = 0.1$~m$^2$.  
The UAV’s maximum speed is $\mathrm{V}_{\max} = 4$~m/s, and the maximum changes in heading angle and speed per time slot are $\Delta\phi = \pi/6$~rad and $\Delta\mathrm{V} = 0.5$~m/s, respectively.  
The variance of the environmental disturbance affecting the UAV’s motion is $\sigma^2_\eta = 0.005$. 
The safety distance to avoid collision is $\mathrm{D}_{\text{safe}}=0.5$ m, while the threshold distance for the UAV to reach its destination is $\mathrm{D}_{\text{thr}}=0.3$ m.
A total of $N_o = 10$ dynamic obstacles are randomly distributed within a square area of $[2,~8] \times [2,~8]$~m$^2$ along the UAV’s flight route.  
Each obstacle moves with horizontal velocity components $V_{ox}(t_i)$ and $V_{oy}(t_i)$ uniformly distributed within $[-1,~1]$~m/s.  
The detection noise for each obstacle has variance $\sigma^2_o = 0.001$.  
The scanning radius for obstacles is $\mathrm{R}_{\text{scan}} = 2$~m.
The BS is located at $\bm{p}_{\text{BS}} = [0,~0]^\top$~m on the ground.  
It is equipped with $\mathrm{K} = 128$ antennas and transmits ISAC signals over $\mathrm{M} = 2500$ subcarriers, with a subcarrier spacing of $\Delta f = 120$~kHz.  
The carrier frequency is $f_c = 60$~GHz, and the antenna element spacing is $\mathrm{D} = 2.5$~mm.  
The transmit power of the BS is $\mathrm{P} = 20$~dBm, while the power spectral density of the AWGN is $\sigma^2_0 = -174$~dBm/Hz.  
The Rician factor of the wireless channel is set to $\kappa = 8$~dB.
Each C\&C signal has a size of $\mathrm{S} = 1$~kbit and is mapped $\mathrm{L} = 50$ times across subcarriers for reliable transmission.  
The discrete angular grid for beamforming contains $N_\vartheta = 500$ points.  
The confidence level factor is $\mathrm{B} = 2.576$, corresponding to a 99\% confidence level. 
The prediction horizon in MD-DWA is set to $\mathrm{I_p} = 20$ time slots.
We conduct simulations over $60$ random seeds, where the initial obstacle positions vary across seeds, and all the following simulation results are average values. 
The parameters of DQN are summarized in \textbf{\tabref{tab}}.

\begin{table}[h]
    \renewcommand\arraystretch{1.3}
    \caption{DQN Parameter Settings}\label{tab}
    \centering
    \begin{tabular}{|p{5cm}<{\centering}|p{2cm}<{\centering}|}
        \hline
        \textbf{Parameter description}       & \textbf{Value}     \\ 
        \hline
        Collision penalty $\Psi$      		 & 10   \\
        \hline
        Discount factor $\gamma$             & 0.9  \\
        \hline
        Exploration probability threshold $\varepsilon$ & 0.8 \\
        \hline
        Memory buffer capacity $\mathrm{C_M}$        & 8000  \\
        \hline
        Sample batch $\mathrm{N}_0$          & 32  \\
        \hline
        Learning rate $\mathrm{L_r}$         & 0.001   \\
        \hline
        Target network update step  $\mathrm{N_U}$      & 100 \\ 
        \hline
        Number of hidden layers          & 2 \\
        \hline
        Number of neurons in each hidden layer     & 128  \\
        \hline
    \end{tabular}
\end{table}

To demonstrate the robustness of the learning component in the E-DQN, we investigate the impact of several representative hyperparameters. The simulation results are shown in \figref{hyper_parameter}. As observed, when the discount factor $\gamma$ decreases to $0.85$, the convergence value becomes lower because the agent places less emphasis on future rewards, leading to a more myopic policy that sacrifices long-term performance. Meanwhile, the training process exhibits improved stability, as reduced sensitivity to long-term reward propagation mitigates oscillations during value updates. When the exploration rate $\varepsilon$ decreases to $0.7$, the convergence value becomes lower in this setting. Although a smaller $\varepsilon$ encourages more exploration, the excessive exploration manifests as a higher probability of collisions with obstacles, resulting in significant penalty accumulation. Consequently, the overall return is reduced and the training stability degrades. Finally, when the learning rate decreases to $0.0005$, convergence becomes slower because parameter updates are more incremental, resulting in reduced learning speed. However, the stability improves and the final convergence value becomes higher, as smaller update steps prevent overshooting and enable more precise approximation of the optimal Q-values.

For comparison, we present the other three ISAC signal transmission schemes as baselines:
\begin{itemize}
    \item \textbf{Traditional sensing and communication signals transmission (Trad-SC)} \cite{continuous}. As discussed in \sectionref{trad}, ISAC signals are transmitted at each time slot. The inflation-based DWA is employed to generate C\&C signals.
    \item \textbf{Periodic sensing and communication signal transmission (P-SC)} \cite{period}. The sensing and C\&C signals are transmitted every 10 time slots. The KF and MD-DWA are used for prediction and C\&C signal generation.
    \item \textbf{Event-trigger sensing and communication signal transmission (ET-SC)} \cite{ET}. This is a commonly method in robotic control system, in which control signals are triggered when the accumulated error exceeds a threshold. We extend the method for sensing and C\&C signals transmission: If no obstacles are detected, signals are transmitted when $\det({\widetilde{\bm\Gamma}}(t_i)) \ge 0.01$; otherwise, they are transmitted when $\det({\widetilde{\bm\Gamma}}(t_i)) \ge 0.001$. The KF and MD-DWA are also adopted in this scheme.
\end{itemize} 

\begin{figure}[t] 
    \centering
    \includegraphics[width=0.45\textwidth]{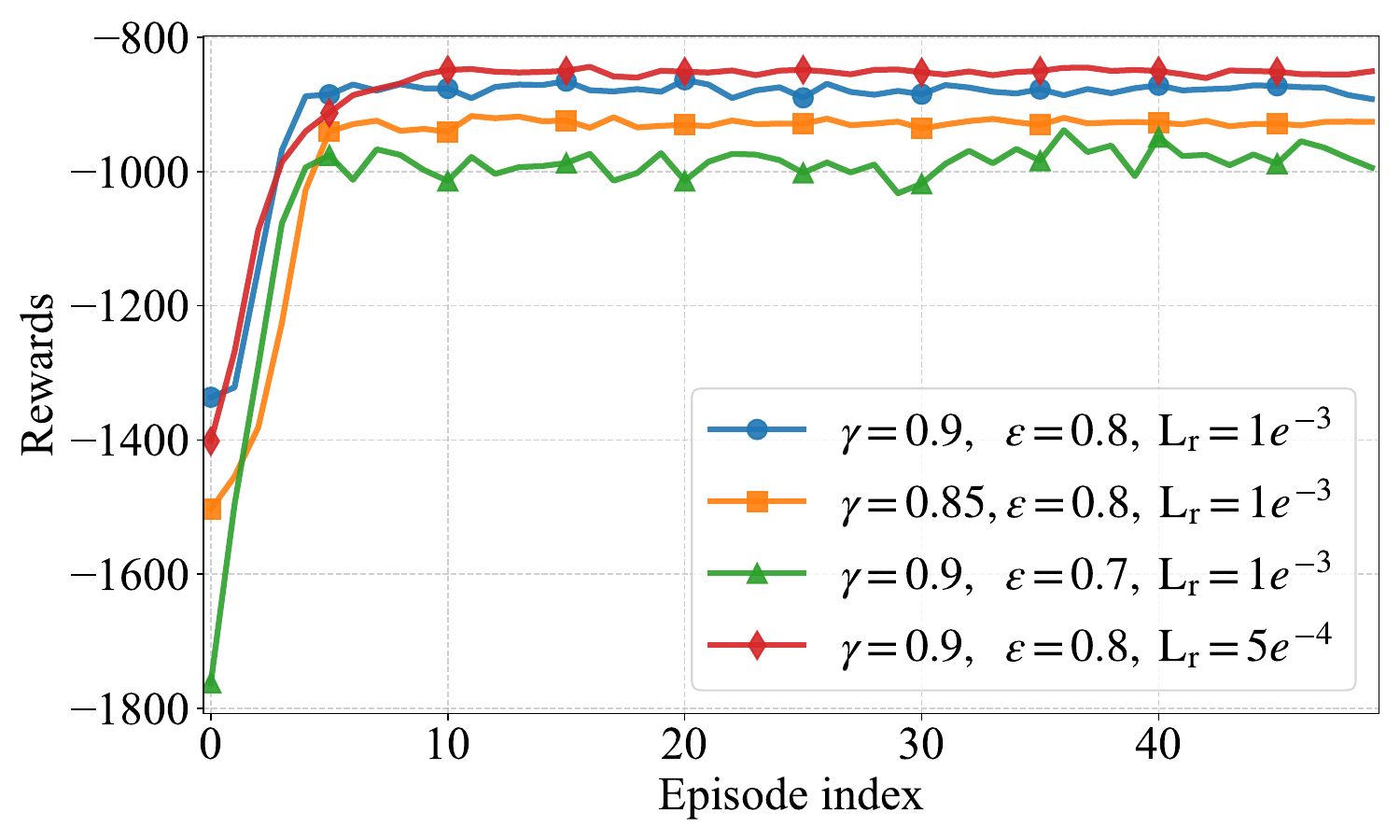}
    \caption{Convergence under different hyperparameters of E-DQN.}
    \label{hyper_parameter}
    \vspace{-0.2 cm}
\end{figure}

\begin{figure}[t] 
    \centering
    \subfloat[][Task success rate.]
    {\includegraphics[width=0.23\textwidth]{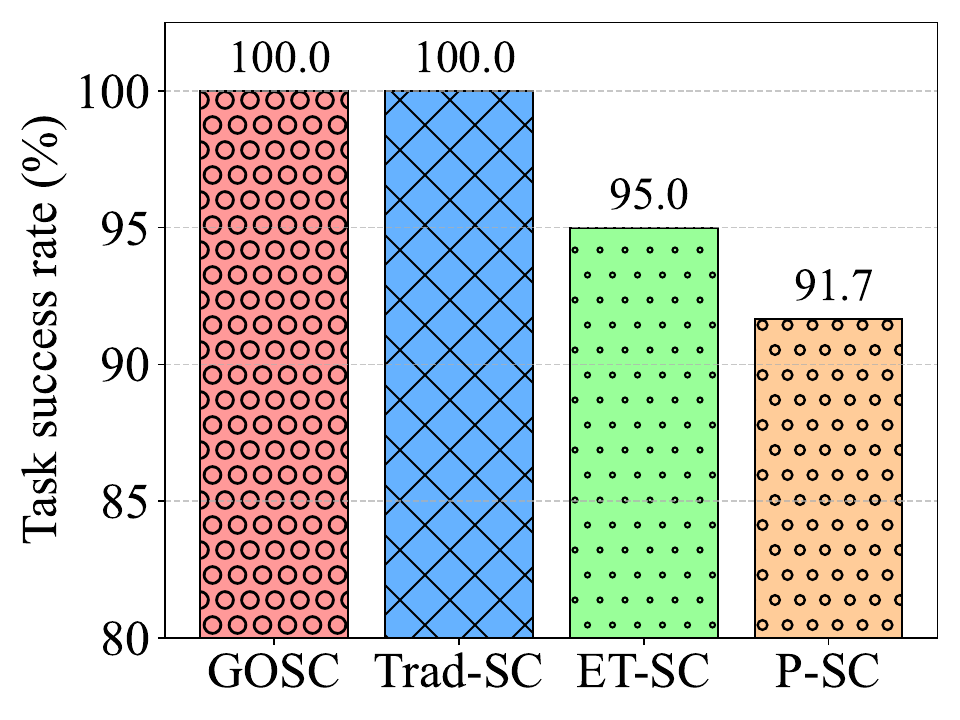}
        \label{success_rate}
    } 
    \subfloat[][Total number of transmitted sensing and C\&C signals.]
    {\includegraphics[width=0.23\textwidth]{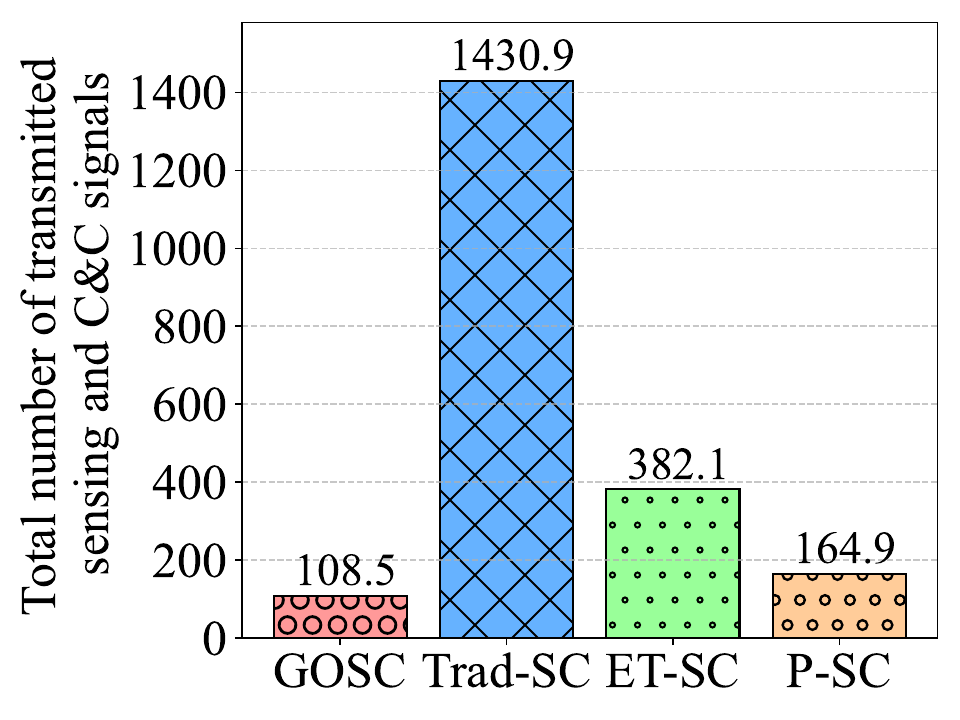}
        \label{total_signal}
    } \hspace{0.01 cm} \\ 
    \subfloat[][Total number of time slots required to complete the task.]
    {\includegraphics[width=0.23\textwidth]{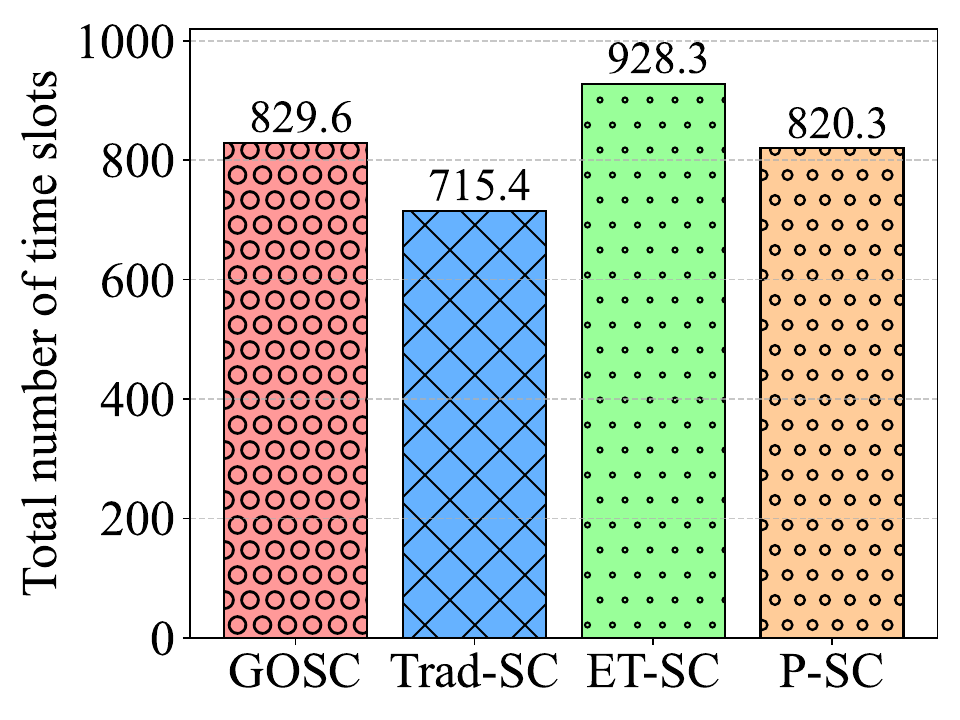}
        \label{total_slot}
    } 
    \subfloat[][UAV flying distance.]
    {\includegraphics[width=0.23\textwidth]{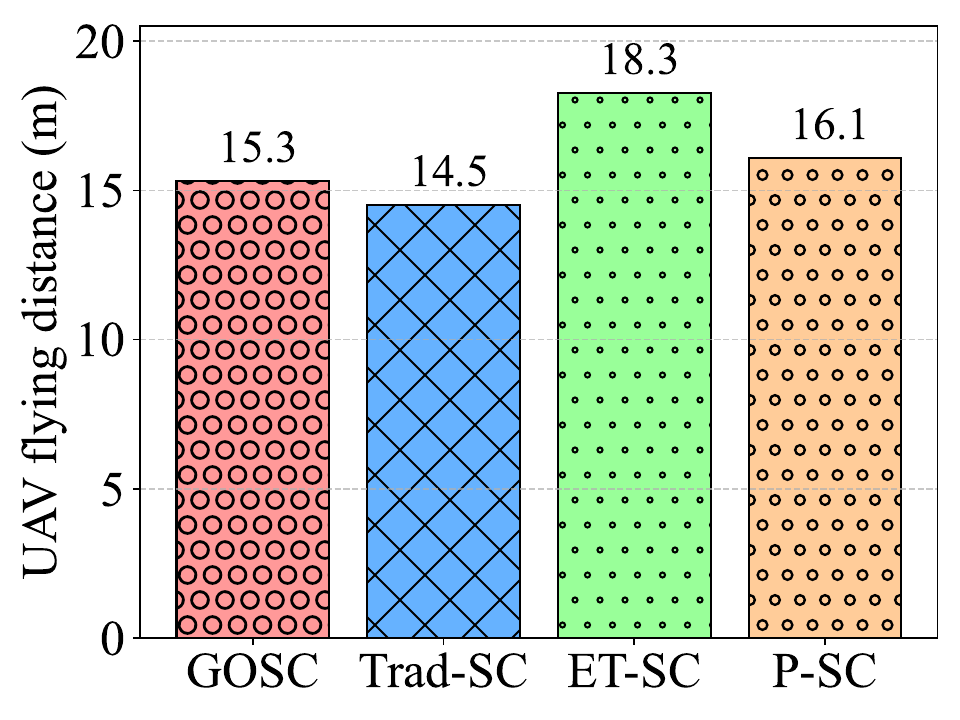}
        \label{path_length}
    } \hspace{0.01 cm} \\
    \subfloat[][The minimum Euclidean distance to obstacles.]
    {\includegraphics[width=0.23\textwidth]{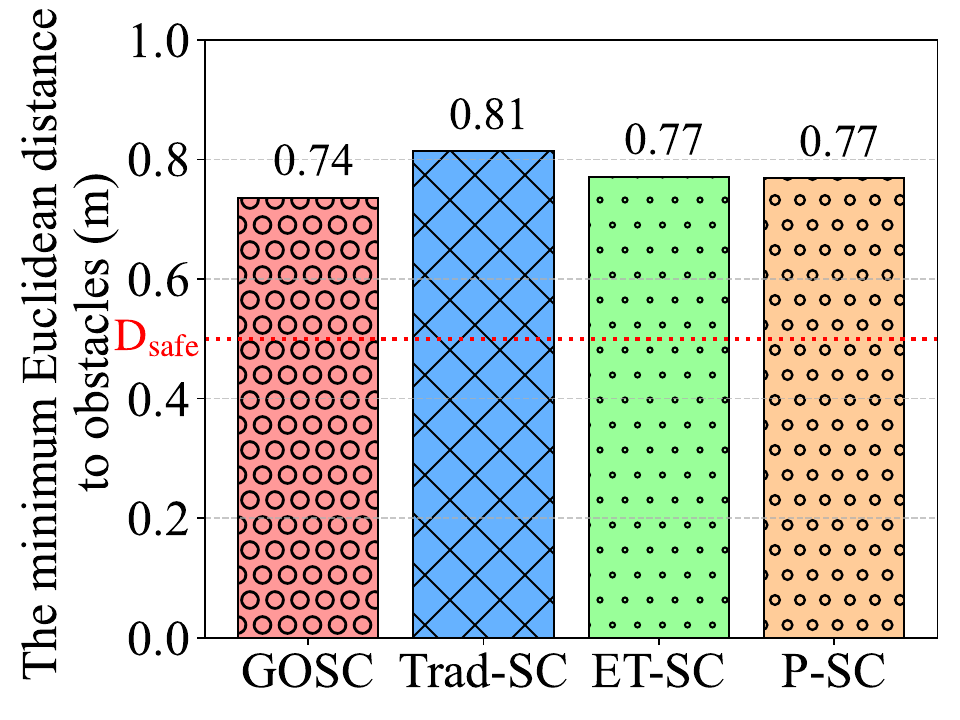}
        \label{mini_dist}
    }
    \subfloat[][Total number of transmission slots.]
    {\includegraphics[width=0.23\textwidth]{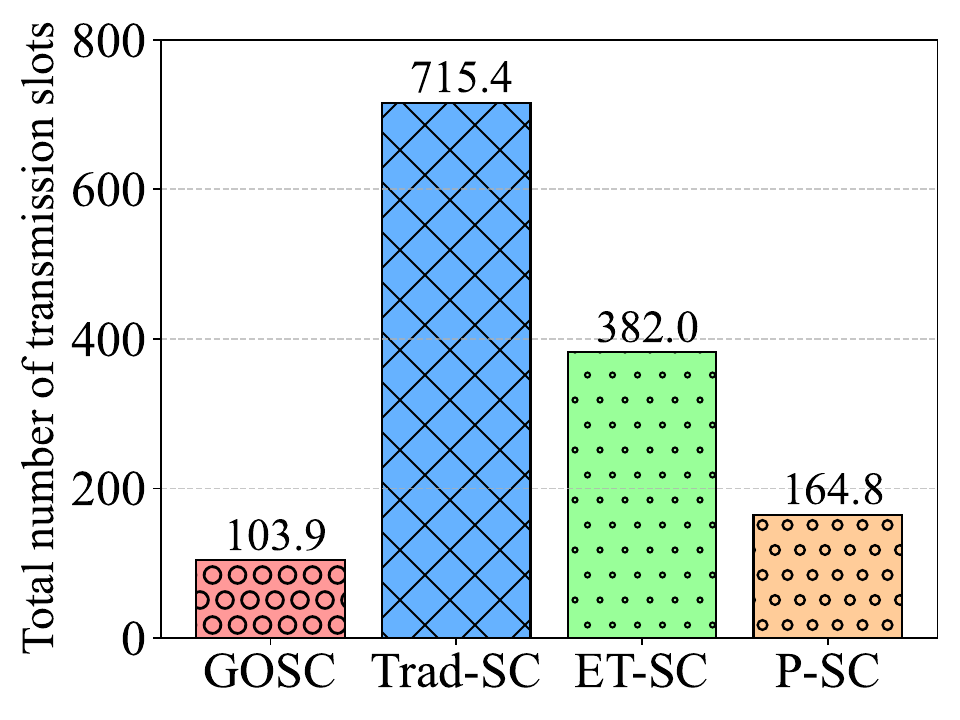}
        \label{trans_slots}
    } 
    \caption{Performances under different schemes.}
    \label{statistics}
    \vspace{-0.3 cm}
\end{figure} 

To examine objective-related performance, \figref{statistics} presents the task success rate, total number of transmitted signals, total number of time slots to complete the task, UAV flying distance, the minimum Euclidean distance to obstacles, and total number of transmission slots under the four schemes. First, the task success rates under different schemes are shown in \figref{success_rate}, representing the probability that the UAV successfully avoids obstacles and reaches its destination. It is no surprising that Trad-SC achieves a 100\% success rate, as ISAC signals are transmitted every time slot. This enables the BS to continuously update the UAV's position and provide real-time control. Remarkably, our proposed GOSC also achieves a 100\% success rate. This can be attributed to the effectiveness-aware DQN, which selectively transmits task-critical signals based on their effectiveness level. In contrast, ET-SC achieves a 95\% success rate. Indeed, its performance is highly sensitive to the signal transmission threshold. A lower threshold increases the likelihood of success but requires the BS to transmit more signals. Finally, P-SC yields the lowest success rate. The main reason is that when the UAV approaches obstacles, essential C\&C signals are needed to guide avoidance maneuvers. However, under P-SC, signal transmission follows a fixed periodic schedule, which may fail to provide vital updates when the UAV is approaching obstacles.

For fair comparison, only statistics from successful tasks are considered in \figref{total_signal}--\figref{trans_slots}. As shown in \figref{total_signal}, GOSC transmits the fewest signals, reducing the number of transmissions by 92.4\% compared to Trad-SC. This advantage stems from the GOSC framework, which can continuously predict the position of UAV and transmit only those effectiveness-relevant signals for safe navigation, thereby eliminating redundant signal transmission. ET-SC and P-SC also reduce the number of transmissions, but this comes at the expense of degraded reliability. Interestingly, when combined with the success rate results of Trad-SC, ET-SC and P-SC in \figref{success_rate}, it can be found that bit-oriented schemes face a distinct tradeoff where higher reliability demands more transmissions, GOSC effectively breaks this limitation, achieving high reliability and low communication costs simultaneously.

\figref{total_slot} and \figref{path_length} present the total number of time slots for the UAV to complete the task and the corresponding flying path length, respectively. Trad-SC achieves the fewest time slots and the shortest flying distance, owing to continuous signal transmission that ensures real-time trajectory update. The flying path length of GOSC is shorter than that of P-SC, while the average task completion time of GOSC is comparable to P-SC. This is because GOSC selectively transmits signals: while the content of C\&C signals remains effective, they are not transmitted every time after sensing to balance transmission cost and task completion time, which may prolong the overall task duration. ET-SC exhibits the worst performance in both metrics. This is due to its threshold-triggered signal transmission, where critical updates may be postponed if the threshold condition is not met, causing inefficient avoidance maneuvers and consequently longer path and completion time.

From \figref{mini_dist}, we observe that the average minimum Euclidean distance of Trad-SC is larger than that of the other three schemes. This is because the inflation-based DWA typically overestimates the obstacle space, resulting in more conservative avoidance maneuvers. In contrast, the MD-DWA provides a more accurate characterization of UAV and obstacles uncertainty, allowing the UAV to maintain a shorter but still safe distance from obstacles. \figref{trans_slots} shows the total number of transmission slots, which represents the time slots during which sensing and/or communication signals are transmitted. This metric reflects the wireless resource allocation cost. It is seen that GOSC reduces the average number of transmission slots by 85.5\% compared with Trad-SC, highlighting its strong capability in saving scarce wireless resources.

\begin{figure}[t] 
    \centering
    \subfloat[][UAV trajectories under GOSC.]
    {\includegraphics[width=0.23\textwidth]{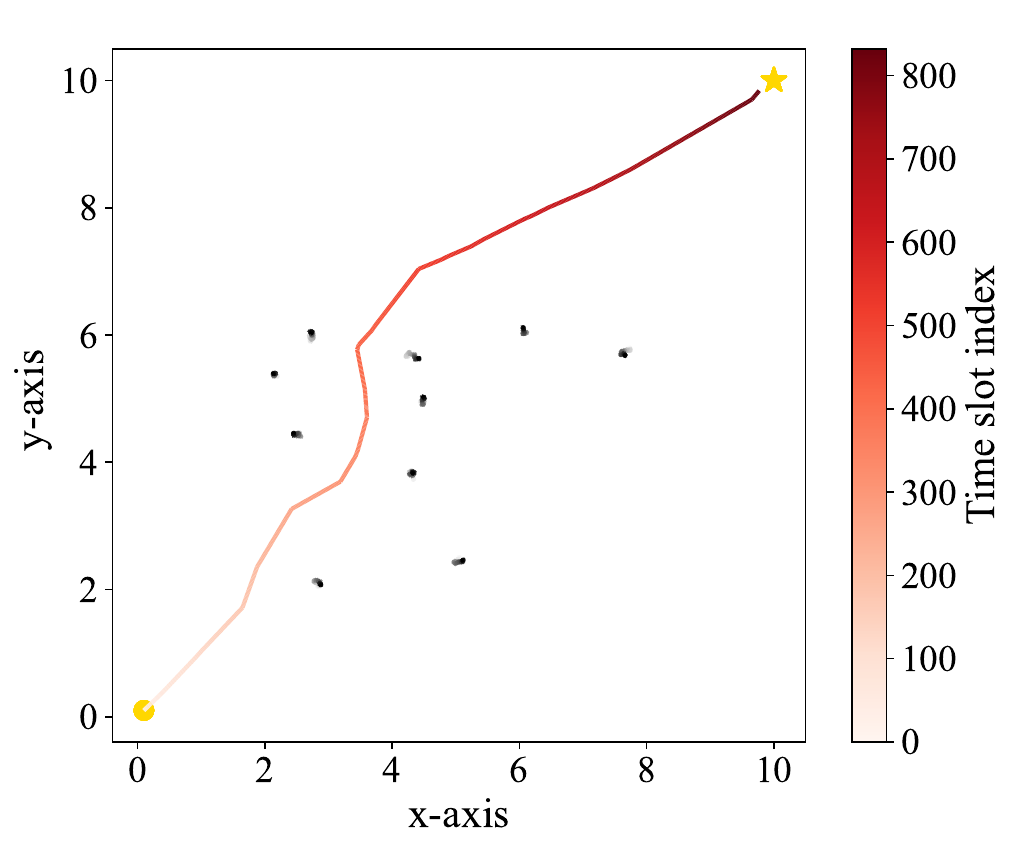}
        \label{GOSC}
    } 
    \subfloat[][UAV trajectories under Trad-SC.]
    {\includegraphics[width=0.23\textwidth]{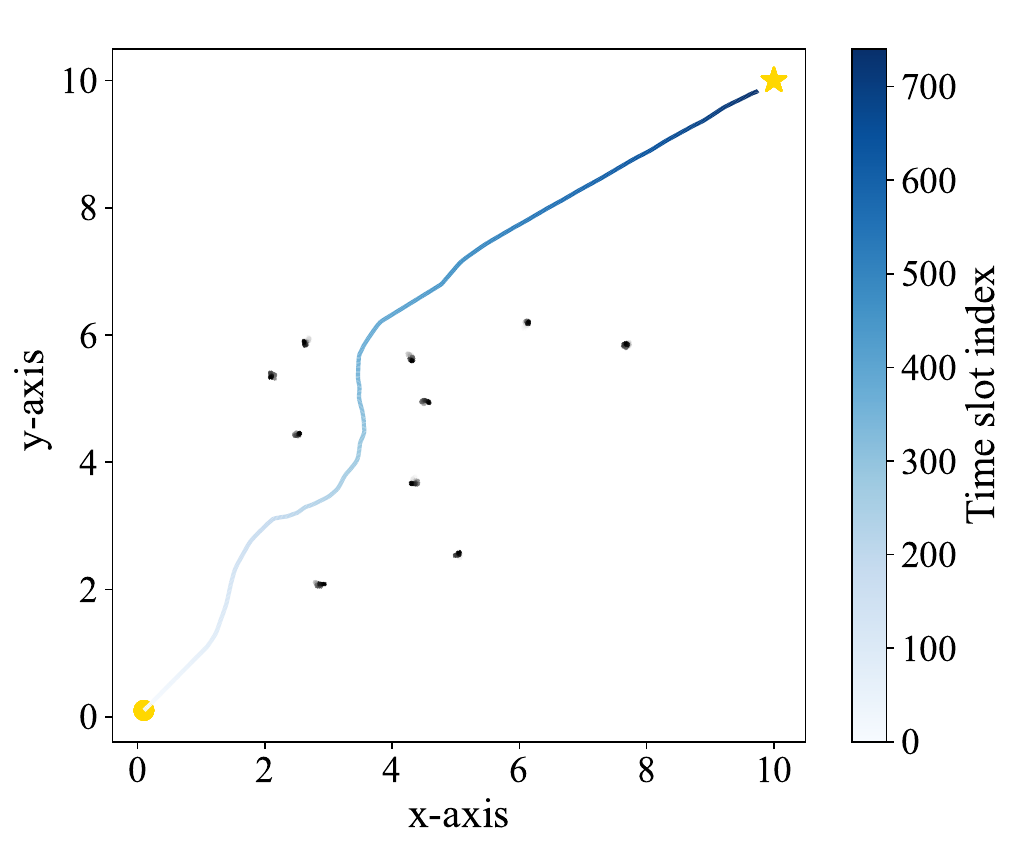}
        \label{Trad-SC}
    } \\
    \subfloat[][UAV trajectories under ET-SC.]
    {\includegraphics[width=0.23\textwidth]{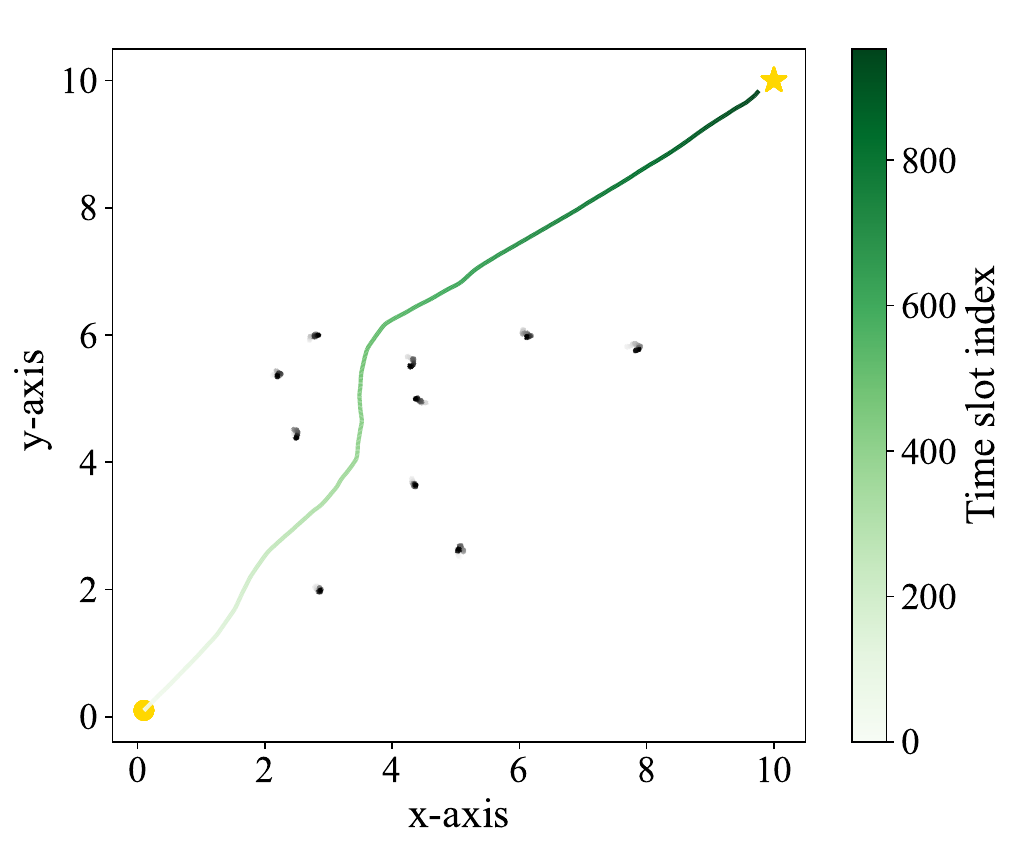}
        \label{ET-SC}
    } 
    \subfloat[][UAV trajectories under P-SC.]
    {\includegraphics[width=0.23\textwidth]{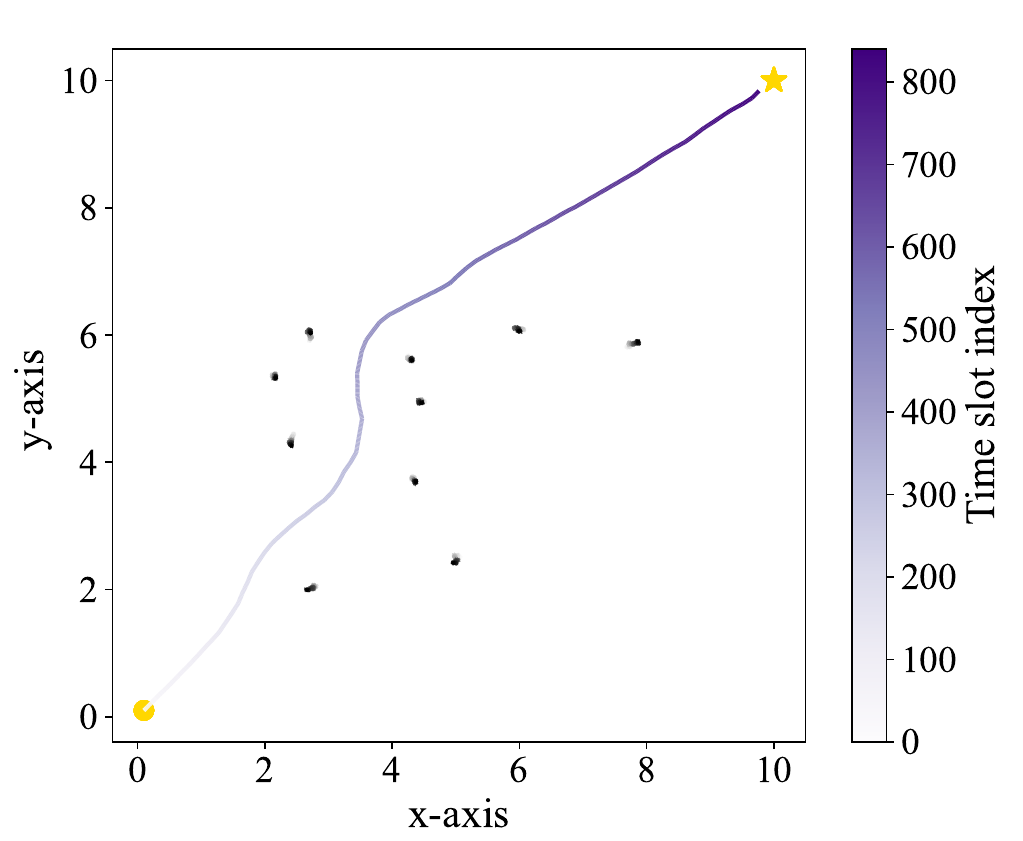}
        \label{P-SC}
    }
    \caption{UAV trajectories under different schemes.}
    \label{four_trajectories}
\end{figure} 

To provide an intuitive understanding of UAV movement in the task, we select a representative random seed and illustrate the UAV trajectories under the four schemes in \figref{four_trajectories}. For clarity, the obstacle detection process is omitted, as it is difficult to visualize in a static figure. We use varying color lines to represent the UAV trajectories over time. The black lines represent the movement of obstacles. 
Although all schemes adopt the DWA algorithm and thus exhibit broadly similar motion patterns, their behaviors differ significantly due to distinct transmission strategies. In Trad-SC transmission scheme, the UAV updates its motion in real time, resulting in the smoothest and the most reactive trajectories, especially in regions with dynamic obstacles. 
In the ET-SC scheme, transmissions occur more frequently when obstacles are detected near the UAV. Consequently, the trajectory is more reactive at the cluttered regions compared to the obstacle-free regions.
In the P-SC scheme, the UAV updates its motion at fixed intervals regardless of the environment. While the trajectory may exhibit a ``staircase'' pattern due to discrete updates, this effect is subtle as the update interval (50 ms, i.e., 10 time slots) is relatively short.
In contrast, the proposed GOSC scheme demonstrates longer linear trajectory segments in obstacle-sparse regions. This is because there is no collision risk at such situation. The BS suppresses unnecessary C\&C transmissions, and the UAV continues executing previously received commands. Such behavior significantly reduces signaling overhead without compromising safety. When the UAV approaches regions with higher obstacle density, the linear trajectory segments become shorter, as the VoI of sensing and C\&C signals increases, vital sensing and C\&C signals must be transmitted to ensure collision avoidance. 

\begin{figure}[t] 
    \centering
    \subfloat[][Performance comparison under different obstacle densities.]
    {\includegraphics[width=0.23\textwidth]{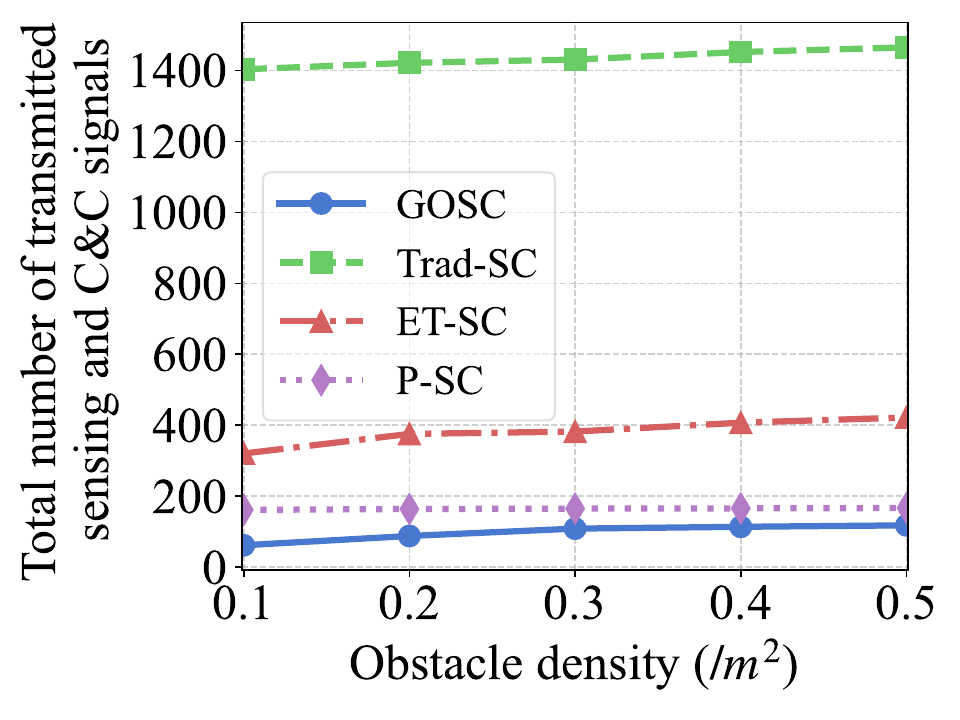}
    \includegraphics[width=0.23\textwidth]{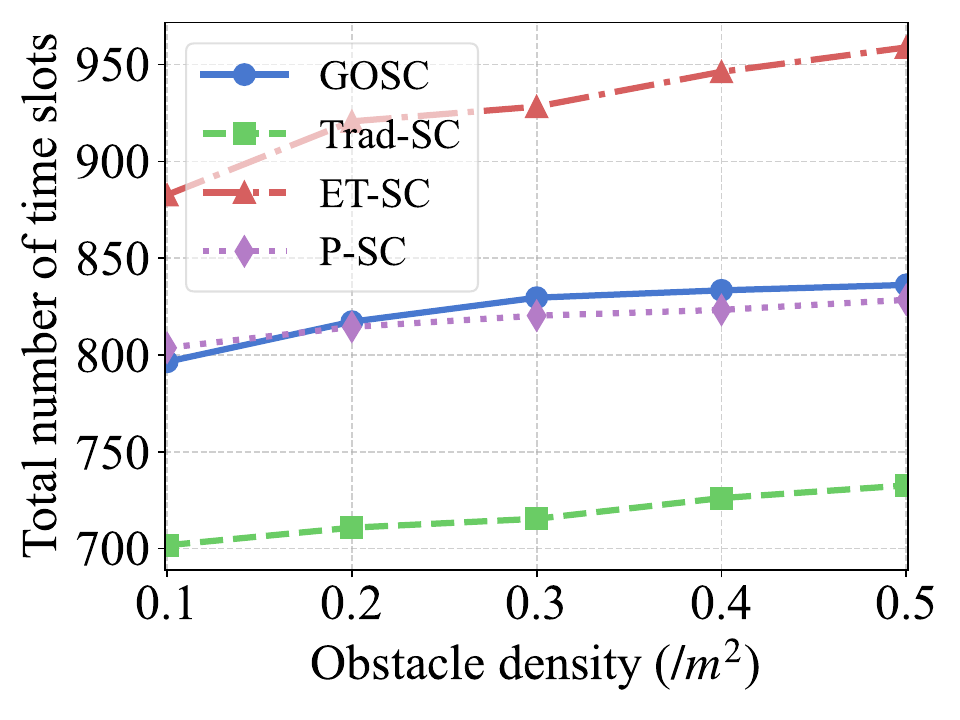}
    \label{obs_density}
    } \hspace{0.01 cm} \\

    \subfloat[][Performance comparison under different UAV maximum speeds.]
    {\includegraphics[width=0.23\textwidth]{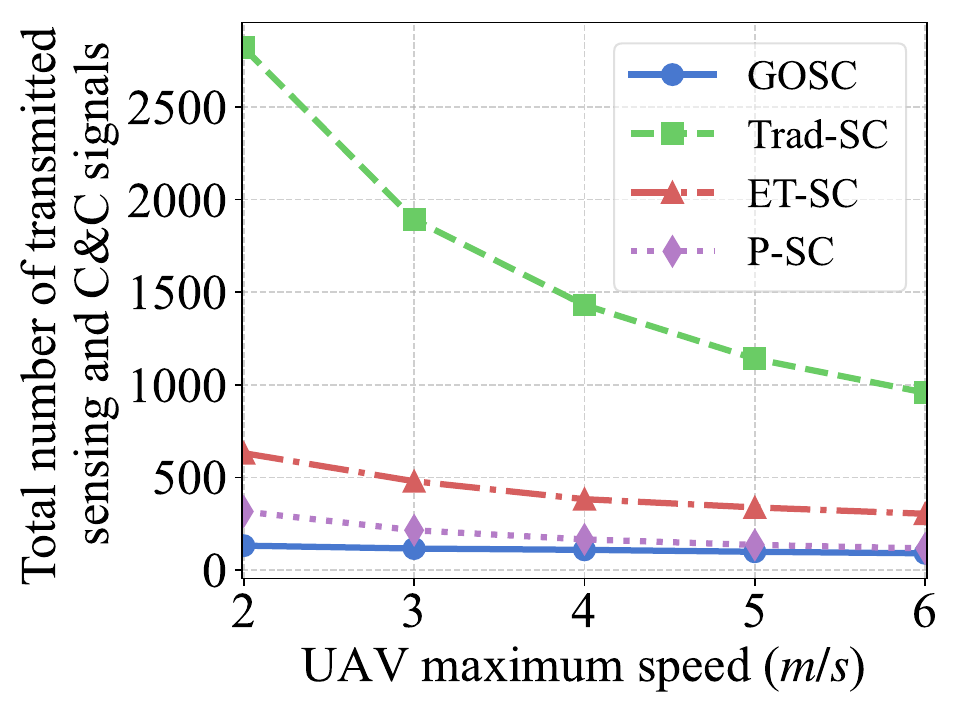}
    \includegraphics[width=0.23\textwidth]{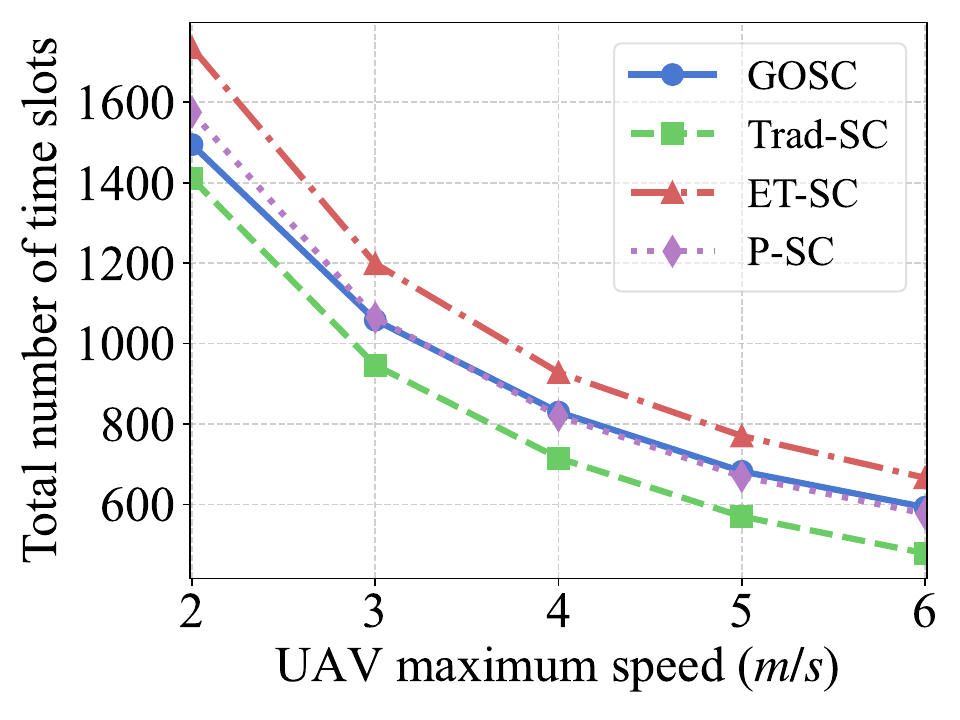}
    \label{uav_speed}
    } \hspace{0.01 cm} \\
    
    \subfloat[][Performance comparison under different bandwidths.]
    {\includegraphics[width=0.23\textwidth]{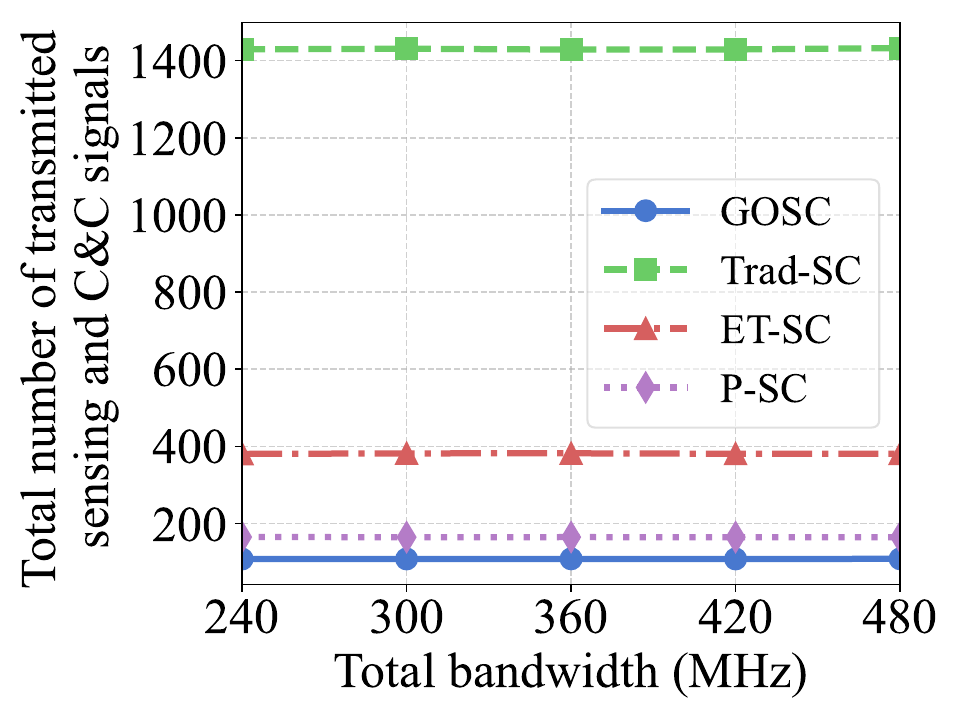}
    \includegraphics[width=0.23\textwidth]{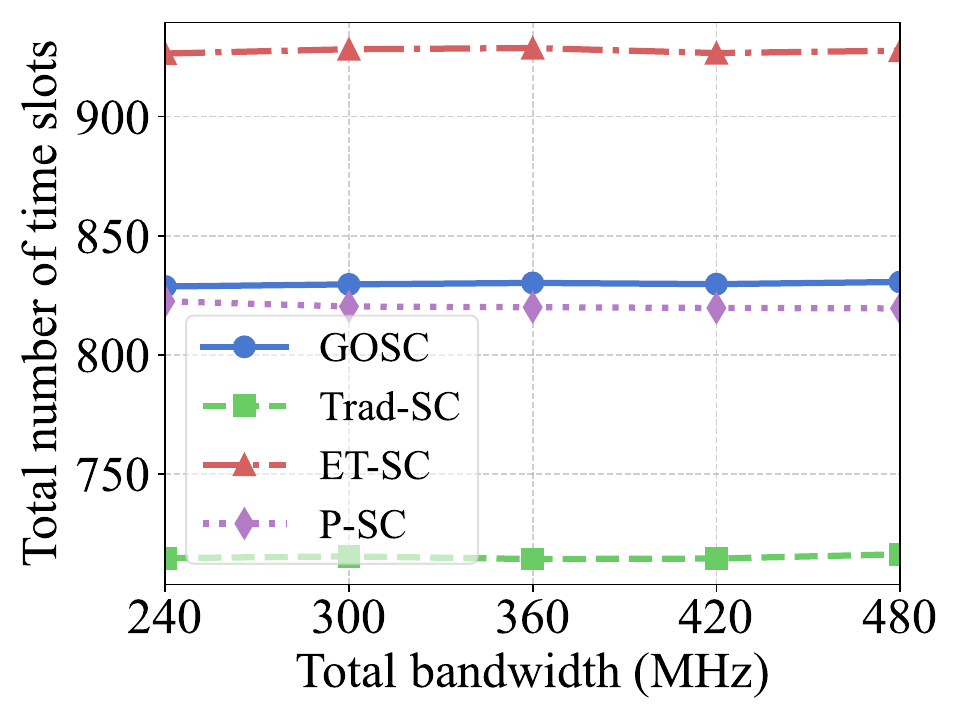}
    \label{bandwidth}
    } \hspace{0.01 cm}
    
    \caption{Performance comparison under different parameter settings.}
    \label{diff_para}
    \vspace{-0.3 cm}
\end{figure} 

To comprehensively evaluate the effectiveness and robustness of the proposed GOSC framework, we investigate the performance of four schemes under different parameter settings, including obstacle density, UAV maximum speed, and system bandwidth, as shown in Fig. 7. It is worth noting that the original training configuration corresponds to an obstacle density of 0.3 /m$^2$ , UAV maximum speed of 4~m/s, and bandwidth of 300~MHz. Under moderate environmental variations (e.g., lower obstacle densities of 0.1 /m$^2$ and 0.2 /m$^2$, as well as all considered UAV maximum speeds and bandwidths), the learned policy can be directly applied without retraining, confirming the robustness and generalization capability of the proposed GOSC framework under moderate variations. However, when the environment becomes more complex (e.g., obstacle densities of 0.4 /m$^2$ and 0.5 /m$^2$), retraining or fine-tuning is required to maintain optimal performance.

\figref{obs_density} presents the performance comparison under varying obstacle densities. As the obstacle density increases, all schemes exhibit an increase in the number of transmitted sensing and C\&C signals, as well as the total time slots, due to the more complex environment requiring more frequent updates. Compared with the baseline schemes, GOSC consistently achieves a significantly lower number of transmissions while maintaining competitive time efficiency. This demonstrates its ability to effectively suppress redundant transmissions by selectively delivering task-relevant information. In contrast, Trad-SC maintains the highest transmission overhead due to its continuous transmission nature, while ET-SC adapts to environmental changes but still incurs higher signaling cost than GOSC. \figref{uav_speed} illustrates the performance under different UAV maximum speeds. As the UAV speed increases, the total number of transmitted signals and the total number of time slots decrease for all schemes, since the UAV reaches the destination more quickly and requires fewer signal updates. GOSC consistently maintains the lowest transmission overhead across all speeds, highlighting its efficiency in dynamic environments. Moreover, the performance gap between GOSC and the baselines becomes more pronounced at lower speeds, where redundant C\&C signals are more likely to occur.

\figref{bandwidth} shows the performance comparison under varying bandwidths. It can be observed that both the total number of transmitted signals and the total number of time slots remain nearly unchanged across different bandwidth values for all schemes. This behavior can be explained from both communication and sensing perspectives. From the communication perspective, increasing bandwidth reduces transmission latency. However, due to the LoS-dominant channel conditions and the small payload size of C\&C data, the transmission delay is already at the microsecond level. As a result, bandwidth variation provides negligible impact in overall latency performance. From the sensing perspective, although increasing bandwidth affects sensing SNR, it induces opposite effects on sensing accuracy: the AoA variance $\sigma^2(\theta(t_i))$ increases, while the range variance $\sigma^2(r(t_i))$ decreases. After coordinate system transformation based on \textit{\textbf{Lemma 1}}, these effects largely offset each other on $\sigma_x^2(t_i)$ and $\sigma_y^2(t_i)$, resulting in nearly identical transmission behavior across different bandwidths. Overall, GOSC consistently outperforms the baseline schemes by achieving the lowest signaling overhead while maintaining comparable task completion time.

\section{Conclusions}  \label{conclusion}
This paper investigated an integrated sensing and communication (ISAC)-enabled BS for the unmanned aerial vehicle (UAV) obstacle avoidance task, and proposed a goal-oriented semantic communication (GOSC) framework to efficiently transmit sensing and command-and-control (C\&C) signals. By integrating a Kalman filter (KF), a Mahalanobis distance-based dynamic window approach (MD-DWA), and an effectiveness-aware deep Q-network (E-DQN), the framework can efficiently transmit sensing and C\&C signals at times when the transmission can benefit for the effectiveness of the robotic task. Specifically, the KF effectively reduces redundant sensing signal transmissions while improving UAV position estimation through sensing–prediction fusion. The MD-DWA generates precise and collision-free C\&C signals by leveraging a mathematically derived minimum Mahalanobis distance. The E-DQN  further enhances communication efficiency by transmitting signals only when their value of information (VoI) is significant.
Extensive simulations demonstrate that, compared to the traditional continuous ISAC transmission framework, our proposed GOSC framework can achieve the same 100\% task success rate while reducing the number of transmitted sensing and C\&C signals by 92.4\% and the required transmission time slots by 85.5\%.

\appendices
\section{Proof of Lemma 1}
For easy description, we omit the index of time slot $t_i$. Define a conversion function from polar coordinate to Cartesian coordinate as
\begin{equation}
            \bm{f}(r,\theta) = \begin{bmatrix}
                x(r,\theta)\\
                y(r,\theta)
            \end{bmatrix} = \begin{bmatrix}
        r\cos(\theta) \\
        r\sin(\theta)
    \end{bmatrix}.
\end{equation}

We approximate $\bm{f}(r,\theta)$ by a first-order Taylor expansion around the nominal values $(\hat{r},\hat{\theta})$, which can be represented as
\begin{equation} \label{taylor}
    \bm{f}(r,\theta) \approx \bm{f}(\hat{r},\hat{\theta}) + \bm{J}(\hat{r},\hat{\theta}) \begin{bmatrix}
        r-\hat{r} \\
        \theta-\hat{\theta}
    \end{bmatrix},
\end{equation}
in which $\bm{J}(\hat{r},\hat{\theta})$ is the Jacobian matrix that obtains the first-order partial derivatives of the function $\bm{f}(r,\theta)$ at $(\hat{r},\hat{\theta})$, which can be written as 
\begin{equation}
    \bm{J}(\hat{r},\hat{\theta}) = \frac{\partial(x,y)}{\partial(r, \theta)}\Bigg|_{(\hat{r},\hat{\theta})} = \begin{bmatrix}
        \cos(\hat{\theta}),&-\hat{r}\sin(\hat{\theta}) \\
        \sin(\hat{\theta}),&\hat{r}\cos(\hat{\theta})
    \end{bmatrix}.
\end{equation}

After resorting \eqref{taylor}, we can obtain
\begin{equation}
    \begin{bmatrix}
        \epsilon_x \\
        \epsilon_y
    \end{bmatrix} \approx \bm{J}(\hat{r},\hat{\theta})
        \begin{bmatrix}
        \epsilon_r \\
        \epsilon_{\theta}
    \end{bmatrix}.
\end{equation}

\section{Proof of Lemma 2}
Let $\bm{z}$ denote the relative position between the UAV and the obstacle at the minimum Mahalanobis distance, we have $\bm{z} \sim \mathcal N(\bm\mu,\bm\Sigma)$, in which
\begin{equation} 
\bm\mu = \begin{bmatrix}
\breve p_x(t_{i+a^\star})- p_{o^\star x} \\ \breve p_y(t_{i+a^\star})- p_{o^\star y} \end{bmatrix}, 
\quad 
\bm\Sigma = \mathring{\bm{\Gamma}}(t_{i+a^\star}) +\bm{\Phi}_{o^\star}(t_i).
\end{equation}

According to equation (7) in \cite{chi}, the squared Mahalanobis distance
$\|\bm z-\bm\mu\|_M^2= (\bm z-\bm\mu)^\top \bm\Sigma^{-1} (\bm z-\bm\mu)$ follows a chi-squared ($\chi^2$) distribution with degrees of freedom equal to the dimension of $\bm z$, i.e., $(\bm z -\bm\mu) \sim \chi^2_2$. For a tolerated collision probability $\varrho$, define the confidence threshold $\varpi = \chi^2_{2,\;1-\varrho}$. The corresponding $(1-\varrho)$-confidence ellipse is given as
\begin{equation}
    \mathcal E_{\varpi} = \big\{ \bm e \in \mathbb R^{2\times 1} : \bm e^\top \bm\Sigma^{-1} \bm e \le \varpi \big\},
\end{equation}
where $\bm e = \bm z -\bm\mu $ has zero mean. The UAV also has a fixed safety distance $\mathrm{D}_{\text{safe}}$. After accounting for this distance, if the confidence ellipse does not contain the origin point, we can determine that the UAV can avoid collision with probability at least $(1-\varrho)$. The condition can be written as 
\begin{equation}\label{cond}
\bm{0} \notin \bm{\mu} + \mathcal E_{\varpi} \oplus \mathcal{D},
\end{equation}
where \(\mathcal{D}=\{\bm u:\|\bm u\|\le \mathrm{D}_{\text{safe}}\}\) and \(\oplus\) denotes the Minkowski sum. Equivalently, \eqref{cond} can be transformed as
\begin{equation}\label{no_origin}
    \min_{\bm e\in\mathcal E_{\varpi},\,\|\bm u\|\le \mathrm{D}_{\text{safe}}} \|\bm\mu + \bm e + \bm u\|_M > 0.
\end{equation}

By the triangle inequality, we have
\begin{equation}
\begin{aligned}
    \min_{\bm e\in\mathcal E_{\varpi},\,\|\bm u\|\le \mathrm{D}_{\text{safe}}} \|\bm\mu + \bm e + \bm u\|_M \ge ~ &\|\bm\mu\|_M - \sup_{\bm e\in\mathcal E_{\varpi}}\|\bm e\|_M \\
    &- \sup_{\|\bm u\|_2\le  \mathrm{D}_{\text{safe}}}\|\bm u\|_M,
\end{aligned}
\end{equation}
where $\sup(\cdot)$ denotes the least upper bound of a set. By definition of $\mathcal E_{\varpi}$, we have $\sup_{\bm e\in\mathcal E_{\varpi}}\|\bm e\|_M=\sqrt{\varpi}$. To bound \(\sup_{\|\bm u\|_2\le  \mathrm{D}_{\text{safe}}}\|u\|_M\), we use the Rayleigh quotient inequality \cite{rayleigh}, which can be written as
\begin{equation}
    \begin{aligned}
        \bm u^\top\bm\Sigma^{-1}\bm u = \frac{\bm u^\top\bm\Sigma^{-1}\bm u}{\bm u^\top \bm u} \bm u^\top \bm u \le \lambda_{\max}(\bm\Sigma^{-1})\|\bm u\|^2 = \frac{\|\bm u\|^2}{\lambda_{\min}(\bm\Sigma)},
    \end{aligned}
\end{equation}
since $\bm\Sigma$ is symmetric positive definite, its smallest eigenvalue satisfies \(\lambda_{\min}(\bm\Sigma)>0\). Taking square roots gives
\begin{equation}
    \|\bm u\|_M \le \frac{\|\bm u\|}{\sqrt{\lambda_{\min}(\bm\Sigma)}}.
\end{equation}

Hence
\begin{equation}
    \sup_{\|\bm u\|_2\le \mathrm{D}_{\text{safe}}}\|\bm u\|_M \le \frac{ \mathrm{D}_{\text{safe}}}{\sqrt{\lambda_{\min}(\bm\Sigma)}}.
\end{equation}

Therefore, a sufficient deterministic condition for \eqref{no_origin} to hold is
\begin{equation}
 \|\bm\mu\|_M > \sqrt{\varpi} + \frac{\mathrm{D}_{\text{safe}}}{\sqrt{\lambda_{\min}(\bm\Sigma)}}.
\end{equation}

\balance
\bibliographystyle{IEEEtran}
\bibliography{IEEEabrv,bib}

\begin{thebibliography}{10}
\providecommand{\url}[1]{#1}
\csname url@samestyle\endcsname
\providecommand{\newblock}{\relax}
\providecommand{\bibinfo}[2]{#2}
\providecommand{\BIBentrySTDinterwordspacing}{\spaceskip=0pt\relax}
\providecommand{\BIBentryALTinterwordstretchfactor}{4}
\providecommand{\BIBentryALTinterwordspacing}{\spaceskip=\fontdimen2\font plus
\BIBentryALTinterwordstretchfactor\fontdimen3\font minus \fontdimen4\font\relax}
\providecommand{\BIBforeignlanguage}[2]{{%
\expandafter\ifx\csname l@#1\endcsname\relax
\typeout{** WARNING: IEEEtran.bst: No hyphenation pattern has been}%
\typeout{** loaded for the language `#1'. Using the pattern for}%
\typeout{** the default language instead.}%
\else
\language=\csname l@#1\endcsname
\fi
#2}}
\providecommand{\BIBdecl}{\relax}
\BIBdecl

\bibitem{OA}
A.~Pandey, S.~Pandey, and D.~Parhi, ``Mobile robot navigation and obstacle avoidance techniques: A review,'' \emph{Int. Robot. Autom. J.}, vol.~2, no.~3, pp. 1--12, May 2017.

\bibitem{sensor_fusion}
J.~Lv, C.~Qu, S.~Du, X.~Zhao, P.~Yin, N.~Zhao, and S.~Qu, ``Research on obstacle avoidance algorithm for unmanned ground vehicle based on multi-sensor information fusion,'' \emph{Math. Biosci. Eng.}, vol.~18, no.~2, pp. 1022--1039, 2021.

\bibitem{OA_alg}
A.~N.~A. Rafai, N.~Adzhar, and N.~I. Jaini, ``A review on path planning and obstacle avoidance algorithms for autonomous mobile robots,'' \emph{J. Robot.}, vol. 2022, no.~1, pp. 1--14, Dec. 2022.

\bibitem{sensor}
X.~Liu, S.~Wen, Z.~Jiang, W.~Tian, T.~Z. Qiu, and K.~M. Othman, ``A multisensor fusion with automatic vision--{LiDAR} calibration based on factor graph joint optimization for {SLAM},'' \emph{IEEE Trans. Instrum. Meas.}, vol.~72, pp. 1--9, Oct. 2023.

\bibitem{OA_AI}
P.~Chen, J.~Pei, W.~Lu, and M.~Li, ``A deep reinforcement learning based method for real-time path planning and dynamic obstacle avoidance,'' \emph{Neurocomputing}, vol. 497, pp. 64--75, Aug. 2022.

\bibitem{ISAC}
F.~Liu, Y.~Cui, C.~Masouros, J.~Xu, T.~X. Han, Y.~C. Eldar, and S.~Buzzi, ``Integrated sensing and communications: Toward dual-functional wireless networks for {6G} and beyond,'' \emph{IEEE J. Sel. Areas Commun.}, vol.~40, no.~6, pp. 1728--1767, Jun. 2022.

\bibitem{reviewer4_1}
A.~Khalili, A.~Rezaei, D.~Xu, and R.~Schober, ``Energy-aware resource allocation and trajectory design for {UAV}-enabled {ISAC},'' in \emph{Proc. IEEE Global Commun. Conf. (GLOBECOM)}, Dec. 2023, pp. 4193--4198.

\bibitem{reviewer4_2}
A.~Khalili, A.~Rezaei, D.~Xu, F.~Dressler, and R.~Schober, ``Efficient {UAV} hovering, resource allocation, and trajectory design for {ISAC} with limited backhaul capacity,'' \emph{IEEE Trans. Wireless Commun.}, vol.~23, no.~11, pp. 17\,635--17\,650, Nov. 2024.

\bibitem{capacity}
J.~Liu, C.~Zhou, M.~Sheng, H.~Yang, X.~Huang, and J.~Li, ``Resource allocation for adaptive beam alignment in {UAV}-assisted integrated sensing and communication networks,'' \emph{IEEE J. Sel. Areas Commun.}, vol.~43, no.~1, pp. 350--363, Jan. 2025.

\bibitem{energy}
C.~Dou, N.~Huang, Y.~Wu, L.~Qian, and T.~Q.~S. Quek, ``Channel sharing aided integrated sensing and communication: An energy-efficient sensing scheduling approach,'' \emph{IEEE Trans. Wireless Commun.}, vol.~23, no.~5, pp. 4802--4814, May 2024.

\bibitem{mse}
J.~Miguel Mateos-Ramos, C.~Häger, M.~Furkan~Keskin, L.~Le~Magoarou, and H.~Wymeersch, ``Model-based end-to-end learning for multi-target integrated sensing and communication under hardware impairments,'' \emph{IEEE Trans. Wireless Commun.}, vol.~24, no.~3, pp. 2574--2589, 2025.

\bibitem{CRB}
Y.~Wang, M.~Tao, and S.~Sun, ``Cram\'{e}r--{R}ao bound analysis and beamforming design for integrated sensing and communication with extended targets,'' \emph{IEEE Trans. Wireless Commun.}, vol.~23, no.~11, pp. 15\,987--16\,000, Nov. 2024.

\bibitem{trade-off}
X.~Jing, F.~Liu, C.~Masouros, and Y.~Zeng, ``{ISAC} from the sky: {UAV} trajectory design for joint communication and target localization,'' \emph{IEEE Trans. Wireless Commun.}, vol.~23, no.~10, pp. 12\,857--12\,872, Oct. 2024.

\bibitem{continuous}
Z.~Lyu, G.~Zhu, and J.~Xu, ``Joint maneuver and beamforming design for {UAV}-enabled integrated sensing and communication,'' \emph{IEEE Trans. Wireless Commun.}, vol.~22, no.~4, pp. 2424--2440, Apr. 2023.

\bibitem{GOSC}
H.~Zhou, Y.~Deng, X.~Liu, N.~Pappas, and A.~Nallanathan, ``Goal-oriented semantic communications for {6G} networks,'' \emph{IEEE Internet Things Mag.}, vol.~7, no.~5, pp. 104--110, Sep. 2024.

\bibitem{JSCC}
E.~Erdemir, T.-Y. Tung, P.~L. Dragotti, and D.~G{\"u}nd{\"u}z, ``Generative joint source-channel coding for semantic image transmission,'' \emph{IEEE J. Sel. Areas Commun.}, vol.~41, no.~8, pp. 2645--2657, Aug. 2023.

\bibitem{Reviewer2}
Z.~Lyu, G.~Zhu, J.~Xu, B.~Ai, and S.~Cui, ``Semantic communications for image recovery and classification via deep joint source and channel coding,'' \emph{IEEE Trans. Wireless Commun.}, vol.~23, no.~8, pp. 8388--8404, Aug. 2024.

\bibitem{SR1}
Z.~Wang, Y.~Deng, and A.~Hamid~Aghvami, ``Goal-oriented semantic communications for avatar-centric augmented reality,'' \emph{IEEE Trans. Commun.}, vol.~72, no.~12, pp. 7982--7995, Dec. 2024.

\bibitem{SR2}
S.~Liu, N.~Li, Y.~Deng, and T.~Q.~S. Quek, ``Goal-oriented semantic communication for wireless visual question answering,'' \emph{IEEE J. Sel. Areas Commun.}, vol.~43, no.~12, pp. 4247--4261, Dec. 2025.

\bibitem{xu}
Y.~Xu, H.~Zhou, and Y.~Deng, ``Task-oriented semantics-aware communication for wireless {UAV} control and command transmission,'' \emph{IEEE Commun. Lett.}, vol.~27, no.~8, pp. 2232--2236, Aug. 2023.

\bibitem{wu}
W.~Wu, Y.~Yang, Y.~Deng, and A.~Hamid~Aghvami, ``Goal-oriented semantic communications for robotic waypoint transmission: The value and age of information approach,'' \emph{IEEE Trans. Wireless Commun.}, vol.~23, no.~12, pp. 18\,903--18\,915, Dec. 2024.

\bibitem{dect_errors}
J.~A. Zhang, F.~Liu, C.~Masouros, R.~W. Heath, Z.~Feng, L.~Zheng, and A.~Petropulu, ``An overview of signal processing techniques for joint communication and radar sensing,'' \emph{IEEE J. Sel. Top. Signal Process.}, vol.~15, no.~6, pp. 1295--1315, Nov. 2021.

\bibitem{LOS}
C.~You and R.~Zhang, ``{3D} trajectory optimization in {Rician} fading for {UAV}-enabled data harvesting,'' \emph{IEEE Trans. Wireless Commun.}, vol.~18, no.~6, pp. 3192--3207, Jun. 2019.

\bibitem{sensing_receiver}
L.~Pucci, E.~Paolini, and A.~Giorgetti, ``System-level analysis of joint sensing and communication based on {5G New Radio},'' \emph{IEEE J. Sel. Areas Commun.}, vol.~40, no.~7, pp. 2043--2055, Jul. 2022.

\bibitem{reciprocal}
J.~T. Rodriguez, F.~Colone, and P.~Lombardo, ``Supervised reciprocal filter for {OFDM} radar signal processing,'' \emph{IEEE Trans. Aerosp. Electron. Syst.}, vol.~59, no.~4, pp. 3871--3889, Aug. 2023.

\bibitem{MUSIC}
Q.~Zhang, ``Probability of resolution of the {MUSIC} algorithm,'' \emph{IEEE Trans. Signal Process.}, vol.~43, no.~4, pp. 978--987, Apr. 1995.

\bibitem{PIFFT}
M.~Gasior and J.~Gonzalez, ``Improving {FFT} frequency measurement resolution by parabolic and {Gaussian} spectrum interpolation,'' in \emph{AIP Conf. Proc.}, vol. 732, no.~1, 2004, pp. 276--285.

\bibitem{MUSIC_dev}
H.~L. Van~Trees, \emph{Optimum array processing: Part {IV} of detection, estimation, and modulation theory}.\hskip 1em plus 0.5em minus 0.4em\relax Hoboken, NJ, USA: Wiley, 2002.

\bibitem{PIFFT_dev}
A.~C. Turlapaty, Y.~Jin, and Y.~Xu, ``Range and velocity estimation of radar targets by weighted {OFDM} modulation,'' in \emph{Proc. IEEE Radar Conf. (RadarConf)}, May 2014, pp. 1358--1362.

\bibitem{first_oder_error}
F.~Gustafsson and G.~Hendeby, ``Some relations between extended and unscented {Kalman} filters,'' \emph{IEEE Trans. Signal Process.}, vol.~60, no.~2, pp. 545--555, Feb. 2012.

\bibitem{FSPL}
T.~S. Rappaport, G.~R. MacCartney, M.~K. Samimi, and S.~Sun, ``Wideband millimeter-wave propagation measurements and channel models for future wireless communication system design,'' \emph{IEEE Trans. Commun.}, vol.~63, no.~9, pp. 3029--3056, Sep. 2015.

\bibitem{MRC}
D.~Ciuonzo, G.~Romano, and P.~Salvo~Rossi, ``Performance analysis and design of maximum ratio combining in channel-aware {MIMO} decision fusion,'' \emph{IEEE Trans. Wireless Commun.}, vol.~12, no.~9, pp. 4716--4728, Sep. 2013.

\bibitem{BF}
J.~Tranter, N.~D. Sidiropoulos, X.~Fu, and A.~Swami, ``Fast unit-modulus least squares with applications in beamforming,'' \emph{IEEE Trans. Signal Process.}, vol.~65, no.~11, pp. 2875--2887, Jun. 2017.

\bibitem{obstacle_dynamics}
C.~Wang, Z.~Wei, W.~Jiang, H.~Jiang, and Z.~Feng, ``Cooperative sensing enhanced {UAV} path-following and obstacle avoidance with variable formation,'' \emph{IEEE Trans. Veh. Technol.}, vol.~73, no.~6, pp. 7501--7516, Jun. 2024.

\bibitem{scan_range}
D.~Huo, L.~Dai, R.~Chai, R.~Xue, and Y.~Xia, ``Collision-free model predictive trajectory tracking control for {UAVs} in obstacle environment,'' \emph{IEEE Trans. Aerosp. Electron. Syst.}, vol.~59, no.~3, pp. 2920--2932, Jun. 2023.

\bibitem{DWA}
M.~Dobrevski and D.~Sko\u{c}aj, ``Dynamic adaptive dynamic window approach,'' \emph{IEEE Trans. Robot.}, vol.~40, pp. 3068--3081, 2024.

\bibitem{inflation}
Z.~Jian, Z.~Liu, H.~Shao, X.~Wang, X.~Chen, and B.~Liang, ``Path generation for wheeled robots autonomous navigation on vegetated terrain,'' \emph{IEEE Robot. Autom. Lett.}, vol.~9, no.~2, pp. 1764--1771, Feb. 2024.

\bibitem{mahala}
H.~Ghorbani, ``Mahalanobis distance and its application for detecting multivariate outliers,'' \emph{Facta Univ., Math. Inform.}, vol.~34, no.~3, pp. 583--595, Jun. 2019.

\bibitem{DQN_nature}
V.~Mnih \emph{et~al.}, ``Human-level control through deep reinforcement learning,'' \emph{Nature}, vol. 518, no. 7540, pp. 529--533, Feb. 2015.

\bibitem{DQN}
P.~Luong, F.~Gagnon, L.-N. Tran, and F.~Labeau, ``Deep reinforcement learning-based resource allocation in cooperative {UAV}-assisted wireless networks,'' \emph{IEEE Trans. Wireless Commun.}, vol.~20, no.~11, pp. 7610--7625, Nov. 2021.

\bibitem{dqn_complexity}
J.~Xu, B.~Ai, and T.~Q.~S. Quek, ``Toward interference suppression: {RIS}-aided high-speed railway networks via deep reinforcement learning,'' \emph{IEEE Trans. Wireless Commun.}, vol.~22, no.~6, pp. 4188--4201, Jun. 2023.

\bibitem{period}
K.~Meng, Q.~Wu, S.~Ma, W.~Chen, K.~Wang, and J.~Li, ``Throughput maximization for {UAV}-enabled integrated periodic sensing and communication,'' \emph{IEEE Trans. Wireless Commun.}, vol.~22, no.~1, pp. 671--687, Jan. 2023.

\bibitem{ET}
X.-M. Zhang, Q.-L. Han, X.~Ge, and B.-L. Zhang, ``Accumulative-error-based event-triggered control for discrete-time linear systems: A discrete-time looped functional method,'' \emph{IEEE/CAA J. Autom. Sin.}, vol.~12, no.~4, pp. 683--693, Apr. 2025.

\bibitem{chi}
G.~Gallego, C.~Cuevas, R.~Mohedano, and N.~García, ``On the {Mahalanobis} distance classification criterion for multidimensional normal distributions,'' \emph{IEEE Trans. Signal Process.}, vol.~61, no.~17, pp. 4387--4396, Sep. 2013.

\bibitem{rayleigh}
S.~P. Boyd and L.~Vandenberghe, \emph{Convex Optimization}.\hskip 1em plus 0.5em minus 0.4em\relax Cambridge, U.K.: Cambridge Univ. Press, 2004.

\end{thebibliography}

\end{document}